\documentclass{article}
\usepackage[preprint]{log_2023}						

\usepackage{booktabs}						
\usepackage{multirow}						
\usepackage{amsfonts}						
\usepackage{graphicx}						
\usepackage{duckuments}						
\usepackage{amsmath}
\usepackage{amssymb}
\usepackage{mathtools}
\usepackage{amsthm}

\theoremstyle{plain}
\newtheorem{theorem}{Theorem}[section]
\newtheorem{proposition}[theorem]{Proposition}
\newtheorem{lemma}[theorem]{Lemma}

\theoremstyle{definition}

\theoremstyle{remark}

\newtheorem*{theorem*}{Theorem}
\newtheorem*{proposition*}{Proposition}
\newtheorem*{lemma*}{Lemma}

\usepackage{caption}
\DeclareMathOperator{\vect}{vec}
\DeclareMathOperator{\spn}{span}
\DeclareMathOperator*{\plim}{plim}
\usepackage{pgf}
\usepackage{tikz}
\usepackage{caption}
\usepackage{subcaption}
\usepackage[misc]{ifsym}
\usepackage[numbers,compress,sort]{natbib}	

\title[Rank Collapse Causes Over-Smoothing and Over-Correlation in Graph Neural Networks]{Rank Collapse Causes Over-Smoothing and Over-Correlation in Graph Neural Networks}

\author[Roth et al.]{%
Andreas Roth~(\Letter)\\
TU Dortmund University, Germany\\
\email{andreas.roth@tu-dortmund.de}\And
Thomas Liebig\\
TU Dortmund University, Germany,\\
Lamarr Institute for Machine Learning\\ and Artificial Intelligence\\
\email{thomas.liebig@tu-dortmund.de}
}

\begin{document}

\maketitle

\begin{abstract}
Our study reveals new theoretical insights into over-smoothing and feature over-correlation in graph neural networks. Specifically, we demonstrate that with increased depth, node representations become dominated by a low-dimensional subspace that depends on the aggregation function but not on the feature transformations. For all aggregation functions, the rank of the node representations collapses, resulting in over-smoothing for particular aggregation functions. Our study emphasizes the importance for future research to focus on rank collapse rather than over-smoothing. Guided by our theory, we propose a sum of Kronecker products as a beneficial property that provably prevents over-smoothing, over-correlation, and rank collapse. We empirically demonstrate the shortcomings of existing models in fitting target functions of node classification tasks. 
\end{abstract}

\section{Introduction}
Despite the great success of Graph Neural Networks (GNNs) for homophilic tasks~\citep{gasteiger2018predict,equibind,10.1145/3534678.3542598,roth2022forecasting}, their performance for more complex and heterophilic tasks is largely unsatisfying~\citep{yan2021two}. 
Two leading causes are over-smoothing~\citep{oono2019graph,nt2019revisiting} and feature over-correlation~\citep{jin2022feature}.
The conditions under which over-smoothing and over-correlation occur and in which cases they can be prevented are not well-understood in the general case.
Some works argue that all node representations converge to a constant state and have discarded the interaction with feature transformations in their analysis~\citep{li2018deeper,keriven2022not,rusch2022graph}.
Other works have proven that over-smoothing only occurs when the feature transformations satisfy some constraints~\citep{oono2019graph,cai2020note,zhou2021dirichlet}, and choosing suitable parameters can even cause an over-separation of node representations~\citep{oono2019graph,cai2020note,zhou2021dirichlet}. Over-correlation refers to all feature columns becoming overly correlated, which was only recently empirically observed~\citep{jin2022feature}. While \citet{jin2022feature} highlighted the occurrence even when over-smoothing is prevented, its theoretical understanding is limited.

This work clears up different views on over-smoothing and provides a theoretical investigation of the underlying reason behind both over-smoothing and over-correlation. We show that common graph convolutions induce invariant subspaces, each demonstrating distinct predefined behaviors. Critically, this behavior only depends on the spectrum of the aggregation function and not on the learnable feature transformations or the initial node features. 
When considering the limit in the linear case, a low-dimensional subspace dominates the results, leading to over-smoothing when this subspace aligns with smooth signals.
More importantly, node representations suffer from rank collapse for all aggregation functions, explaining the effect of over-correlation. 
This severely limits the expressivity of deep GNNs, as these models cannot fit target functions that require multiple signals to be amplified, which we empirically confirm also for the non-linear case.
We propose a simple but provably more expressive family of models based on the sum of Kronecker products (SKP) that can maintain the rank of its features.
We summarize our key contributions as follows:
\begin{itemize}
\item We provide a shared theoretical background of over-smoothing and over-correlation by showing the dominance of a fixed subspace induced by the aggregation function, allowing us to provide general proofs for both phenomena.
\item While our insights confirm that over-smoothing can be solved by choosing an appropriate aggregation function, we identify the phenomenon of a rank collapse of node representations as the crucial challenge for GNNs going forward.
\item To counteract this limited expressivity, we propose the utilization of a sum of Kronecker products (SKP) as a general property that provably solves rank collapse, which we empirically confirm.
\end{itemize} 

\section{Preliminaries}


\paragraph{Notation} We consider a graph $G=(\mathcal{V},\mathcal{E})$ consisting of a set of $n$ nodes $\mathcal{V}=\{v_1,\dots,v_n\}$ and a set of edges $\mathcal{E}$. The adjacency matrix $\mathbf{A}\in\{0,1\}^{n\times n}$ has binary entries indicating whether an edge between two nodes exists or not. We assume graphs to be irreducible and aperiodic.
We denote the set of nodes neighboring node $v_i$ as $\mathcal{N}_i = \{v_j | a_{ij} = 1\}$. The degree matrix $\mathbf{D}\in\mathbb{N}^{n\times n}$ is a diagonal matrix with each entry $d_{ii}=|\mathcal{N}_i|$ representing the number of neighboring nodes. 
For a given matrix $\mathbf{M}$, we denote its eigenvalues with $\lambda_1^\mathbf{M},\dots,\lambda_n^{\mathbf{M}}$ that are sorted with decreasing absolute value $|\lambda_1^\mathbf{M}|\geq\dots\geq|\lambda_n^{\mathbf{M}}|$.
A matrix is vectorized $vec(\mathbf{M})$ by stacking its columns into a single vector. The identity with $d$ dimensions is denoted by $\mathbf{I}_d\in\mathbb{R}^{d\times d}$.

\paragraph{Graph neural networks} Given a graph $G$ and a graph signal $\mathbf{X}\in\mathbb{R}^{n\times d}$ representing $d$ features at each node, the goal of a graph neural network (GNN) is to find node representations that can be used effectively within node classification, edge prediction or various other challenges. Within these, graph convolutional operations repeatedly update the state of each node by combining the state of all nodes with the states from their neighbors~\citep{gilmer2017neural}.

We consider graph convolutions that iteratively transform the previous node states of the form
\begin{equation}
\label{eq:update}
    \mathbf{X}^{(k+1)} = \Tilde{\mathbf{A}}\mathbf{X}^{(k)}\mathbf{W}^{(k)}\, .
\end{equation}
At each iteration $k$, we consider distinct feature transformations $\mathbf{W}^{(k)}\in\mathbb{R}^{d\times d}$ of the node representations and a homogeneous neighbor aggregation represented by $\Tilde{\mathbf{A}}\in\mathbb{R}^{n\times n}$. Popular instantiations covered by this notation include the graph convolutional network (GCN)~\citep{kipf2017semi}, the graph isomorphism network (GIN)~\citep{xu2018powerful}, and the graph attention network (GAT)~\citep{velivckovic2017graph} with a single head.

\section{Related Work}
\label{sec:related}

\paragraph{Over-smoothing in GNNs}
Over-smoothing arises when node representations $\mathbf{X}^{(k)}$ exhibit excessive similarity to one another as the number of layers $k$ increases. 
As our study extends prior analyses, we first outline the presently available theoretical insights concerning over-smoothing when considering the linear case.
\citet{li2018deeper} connected over-smoothing in GCNs to a special form of Laplacian smoothing
when ignoring the feature transformation. 
While some methods similarly ignore the feature transformation~\citep{rusch2022graph,keriven2022not}, its role is more complicated. 
The pioneering work of \citet{oono2019graph} showed that the feature transformation must not be ignored by bounding the distance
\begin{equation}
    d_\mathcal{M}(\mathbf{AX}\mathbf{W}) \leq \lambda_2^\mathbf{A}\sigma_1^\mathbf{W}d_\mathcal{M}(\mathbf{X})
\end{equation}
of the representations to a smooth subspace $\mathcal{M}$ that is induced by the dominant eigenvector $\mathbf{v}_1$. The bound uses the second largest eigenvalue $\lambda_2^\mathbf{A}$ and the largest singular value $\sigma_1^\mathbf{W}$ of $\mathbf{W}$. Intuitively, each aggregation step $\mathbf{A}$ reduces this distance, while $\mathbf{W}$ can increase the distance arbitrarily. Thus, they consequently claim potential over-separation when $\sigma_1^\mathbf{W}>\frac{1}{\lambda_2^\mathbf{A}}$, which refers to node representations differing strongly.
As an interpretable metric to determine the smoothness of a graph signal, \citet{cai2020note} introduced the Dirichlet energy
\begin{equation}
    E(\mathbf{X}) = tr(\mathbf{X}^T\mathbf{\Delta X}) = \frac{1}{2}\sum_{(i,j)\in \mathcal{E}}\Big\lVert\frac{\mathbf{x}_i}{\sqrt{d_i}}-\frac{\mathbf{x}_j}{\sqrt{d_j}}\Big\rVert^2_2
\end{equation}
using the symmetrically normalized graph Laplacian $\mathbf{\Delta} = \mathbf{I}_n-\mathbf{D}^{-\frac{1}{2}}\mathbf{A}\mathbf{D}^{-\frac{1}{2}}$.
A low energy value corresponds to similar node states.
Similarly to \citet{oono2019graph}, they provided the bound
\begin{equation}
\label{eq:dir_bound}
    E(\mathbf{AXW}) \leq \left(\lambda_2^\mathbf{A}\right)^2\left(\sigma_1^\mathbf{W}\right)^2E(\mathbf{X})
\end{equation}
for each convolution and prove an exponential convergence in the limit. 
As their proof again only holds in case $\sigma_1^\mathbf{W}\leq 1/\lambda_2^\mathbf{A}$, they similarly claim potential over-separation.
\citet{zhou2021dirichlet} provide a lower bound on the energy to show that the Dirichlet energy can go to infinite. 
All of these works conclude that over-smoothing only occurs with high probability and propose to find suitable feature transformations $\mathbf{W}$ to trade-off between over-smoothing and over-separation~\citep{oono2019graph,cai2020note,chen2020measuring,zhou2021dirichlet}.
\citet{di2022graph} and \citet{maskey2023fractional} studied the case when all feature transformations are restricted to be the same symmetric matrix.
It remains unclear which feature transformations reduce or increase the Dirichlet energy in the unconstrained case.
Another line of work found node representations to converge to a constant state~\cite{rusch2022graph,rusch2022gradient,rusch2023survey,wu2023demystifying}, though it remains unclear for which aggregation functions and feature transformations this holds.
While various methods to mitigate over-smoothing have been proposed, the underlying issue persists~\cite{Rong2020DropEdge,Zhao2020PairNorm,chen2020simple,roth2023transforming}.



\paragraph{Over-correlation in GNNs}
\citet{jin2022feature} empirically revealed an excessive correlation among node features with increased model depth, as measured by the Pearson correlation coefficient. This over-correlation results in redundant information and limits the performance of deep GNNs. 
The study emphasizes that over-smoothing always leads to over-correlation, but that over-correlation also arises independently.  Similarly, \citet{guo2022contranorm} empirically observed that feature columns in GNNs and Transformers have the tendency to converge to non-smooth states characterized by low effective rank, where most singular values are close to zero. This phenomenon, which they refer to as \textit{dimensional collapse}, has also been observed in the context of contrastive learning~\cite{DBLP:conf/iclr/GaoHTQWL19,jing2022understanding}. To address over-correlation, some works suggest incorporating it into the loss function~\citep{jin2022feature,guo2022orthogonal}, while others propose combining the results of message-passing with decorrelated signals from the previous state~\citep{guo2022contranorm,liu2023enhancing}. 
However, the underlying theory behind over-correlation remains largely unexplored, especially concerning feature transformations. 

\section{Over-Smoothing and Convergence to a Constant State}

\begin{figure}
     \centering
     \begin{subfigure}[b]{0.44\textwidth}
         \centering
        \def\svgwidth{\textwidth}
         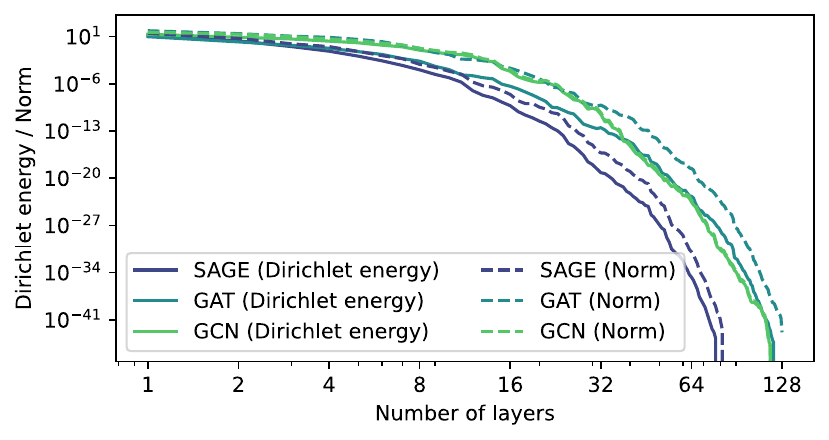
         \caption{Dirichlet energy converges to zero.}
     \end{subfigure}
     \hfill
     \begin{subfigure}[b]{0.44\textwidth}
         \centering
        \def\svgwidth{\textwidth}
         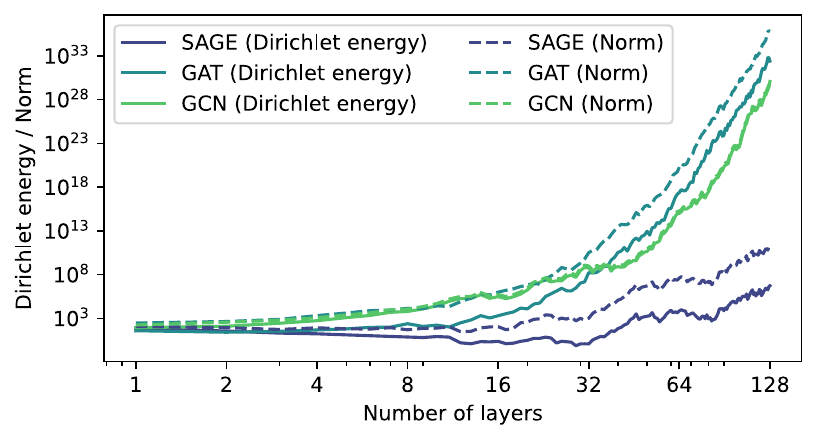
         \caption{Dirichlet energy does not converge.}
     \end{subfigure}
        \caption{Comparison of the Dirichlet energy $E_{rw}(\mathbf{X}^{(l)})$ and the norm $||\mathbf{X}^{(l)}||_F^2$ of the representations after $l$ layers for three commonly used models (Graph Convolutional Network (GCN)~\cite{kipf2017semi}, Graph Attention Network (GAT)~\cite{velivckovic2017graph}, GraphSAGE~\cite{hamilton2017inductive}) on the Cora dataset~\cite{mccallum2000automating}. For (a), we used random weights, for (b), we scaled these weights by a factor of $2$.}
        \label{fig:constant}
\end{figure}

We start by clarifying distinctions in the current comprehension of over-smoothing. Recent research~\cite{rusch2022gradient,rusch2022graph,rusch2023survey,wu2023demystifying} defines over-smoothing as the exponential convergence to a constant state using the Dirichlet energy $E_{rw}(\mathbf{X}) = tr(\mathbf{X}^T\mathbf{\Delta}_{rw}\mathbf{X})$ based on the random walk Laplacian $\mathbf{\Delta}_{rw}= \mathbf{I}_n - \mathbf{D}^{-1}\mathbf{A}$ as the constant state corresponds to its nullspace. In contrast, other studies provided theoretical insights into convergence toward the dominant eigenvector of the aggregation function~\cite{oono2019graph,cai2020note,zhou2021dirichlet}, which is not always constant, e.g., for the GCN~\cite{kipf2017semi}. 
We attribute these differences to the norm of $\mathbf{X}^{(k)}$ overshadowing insights from the Dirichlet energy.
Analogous to the Dirichlet energy (see Eq.~\ref{eq:dir_bound}), the norm is likewise bounded by the largest singular value of the feature transformation:
\begin{proposition} (Node representations vanish.)
    Let $\mathbf{\Tilde{A}}\in\mathbb{R}^{n\times n}=\mathbf{D}^{-\frac{1}{2}}\mathbf{A}\mathbf{D}^{-\frac{1}{2}}$, $\mathbf{W}\in\mathbb{R}^{d\times d}$ be any matrix with maximum singular value $\sigma_1^{\mathbf{W}}$, and $\phi$ a component-wise non-expansive mapping satisfying $\phi(\mathbf{0}) = \mathbf{0}$. Then,
    \begin{equation}
        ||\phi(\mathbf{\Tilde{A}XW})||_F \leq \sigma_1^{\mathbf{W}}\cdot ||\mathbf{X}||_F\, .
    \end{equation} 
\end{proposition}
We present all proofs in Appendix~\ref{sec:proofs}.
When $\sigma_1^{\mathbf{W}^{(k)}}<1$ for all $\mathbf{W}^{(k)}$ this implies convergence to the zero matrix.
With all node states close to zero, the Dirichlet energy becomes also zero.


We compare the Dirichlet energy and the norm in Figure~\ref{fig:constant}, and observe a close alignment between both metrics. While we confirm that the Dirichlet energy converges to zero, our findings show that this solely results from the norm converging to zero.  
Repeating this experiment for scaled feature transformations to prevent vanishing norms, we observe the Dirichlet energy also avoids convergence to zero, as shown in Figure~\ref{fig:constant}. This elucidates deviations between directions claiming either convergence of the GCN to constant columns~\cite{rusch2022gradient,rusch2022graph,rusch2023survey} or values proportional to each node's degree~\cite{oono2019graph,cai2020note,zhou2021dirichlet}.

Since the norm of the node representations overshadows the Dirichlet energy, it seems insufficient to assess the energy of the unnormalized state. Other metrics, such as MAD~\cite{chen2020measuring} and SMV~\cite{liu2020towards}, already incorporate feature normalization to quantify over-smoothing. Furthermore, recent studies already consider the Dirichlet energy of the normalized state for studying over-smoothing~\cite{di2022graph,maskey2023fractional}.
Our insights prompt a deeper exploration of over-smoothing and its interaction with feature transformations, aiming to comprehend its causes better, mitigate its impacts, and advance the development of more powerful GNNs.

\section{Understanding Graph Convolutions using Invariant Subspaces}
\label{sec:explaining}
We delve into the underlying cause for over-smoothing and over-correlation while analyzing the role of feature transformations in these phenomena.
We consider the linearized update function from Eq.~\ref{eq:update}
employing any symmetric aggregation function $\Tilde{\mathbf{A}}\in\mathbb{R}^{n\times n}$ 
and time-inhomogeneous feature transformations $\mathbf{W}^{(k)}\in\mathbb{R}^{d\times d}$ for each layer $k$.
Leveraging the vectorized form
\begin{equation}
\vect(\Tilde{\mathbf{A}}\mathbf{X}^{(k)}\mathbf{W}^{(k)}) = (\mathbf{W}^{(k)^T}\otimes\Tilde{\mathbf{A}})\vect(\mathbf{X}^{(k)})= \mathbf{T}^{(k)}\vect(\mathbf{X}^{(k)})
\end{equation}
allows us to combine the aggregation and transformation steps through the Kronecker product $\mathbf{T}^{(k)}=(\mathbf{W}^{(k)^T}\otimes\Tilde{\mathbf{A}})\in\mathbb{R}^{nd\times nd}$, which is frequently used in theoretical studies~\cite{di2022graph,roth2022transforming,maskey2023fractional}. Further details can be found in Appendix~\ref{sec:proofs}.
For clarity, we will omit the transpose.
 
Initially, we demonstrate that for a fixed aggregation $\Tilde{\mathbf{A}}$, all transformations $\mathbf{T}^{(k)}$ induce the same invariant subspaces that only depend on the eigenvectors of $\Tilde{\mathbf{A}}$.
Formally, the vectorization operation enables the decomposition of vectorized node representations 
\begin{equation}
    \mathbf{T}^{(k)}\vect(\mathbf{X}^{(k)}) = \mathbf{T}^{(k)}\mathbf{S}\mathbf{c} = \sum_{i=1}^m \mathbf{T}^{(k)}\mathbf{S}_{(i)} \mathbf{c}_{(i)}
\end{equation}
into a linear combination $\mathbf{c}\in\mathbb{R}^{nd}$ of basis vectors $\mathbf{S}\in\mathbb{R}^{nd\times nd}$. We further split these into a sum of $n$ invariant components across disjoint subspaces $\mathcal{Q}_i = \spn(\mathbf{S}_{(i)}) \subset\mathbb{R}^{nd}$ with their direct sum $\bigoplus_{i=1}^n \mathcal{Q}_i = \mathbb{R}^{nd}$ covering the entire space.
The linearity of $\mathbf{T}^{(k)}$ allows us to apply the transformation on each subspace separately. 
We construct our bases as $\mathbf{S}_{(i)}=\mathbf{I}_d\otimes\mathbf{v}_i\in\mathbb{R}^{n \times d}$, utilizing the eigenvectors $\mathbf{v}_i$ of $\Tilde{\mathbf{A}}$, as these are invariant to any $\mathbf{T}^{(k)}$:
\begin{lemma}
\label{pr:symm_inv} (The subspaces are invariant to any $\mathbf{T}^{(k)}$.)
Let $\mathbf{T} = \mathbf{W}\otimes \Tilde{\mathbf{A}}$ with $\tilde{\mathbf{A}}\in\mathbb{R}^{n\times n}$ symmetric with eigenvectors $\mathbf{v}_1,\dots,\mathbf{v}_n$ and $\mathbf{W}\in\mathbb{R}^{d\times d}$ any matrix. 
Consider the subspaces $\mathcal{Q}_i = \spn(\mathbf{I}_d\otimes \mathbf{v}_i)$ for $i\in[n]$. Then, 
\begin{equation*}
    \forall\ i\in [n]\colon \mathbf{z} \in \mathcal{Q}_{i}\Rightarrow \mathbf{T}\mathbf{z} \in \mathcal{Q}_{i}\,.
\end{equation*}
\end{lemma}
This discovery is pivotal in our investigation, enabling us to dissect each subspace individually and relate their differences.




\subsection{Relating Dynamics in Subspaces}
Our investigation now delves into the effect $\mathbf{T}^{(k)}$ has on each these subspaces. 
Previous work considered the impact of graph convolutions on coefficients $\mathbf{c}$, deriving coarse bounds based on the singular values of $\mathbf{W}$ were found~\cite{oono2019graph,cai2020note,zhou2021dirichlet} (see Section~\ref{sec:related}).  
We instead analyze how $\mathbf{T}^{(k)}$ alters the basis vectors $\mathbf{S}_{(i)}$ of each subspace while maintaining the coefficients $\mathbf{c}$ constant, leading to a more streamlined analysis of the underlying process.
When applying only the aggregation function $\Tilde{\mathbf{A}}$, each basis $\mathbf{S}_{(i)}$ gets scaled by the corresponding eigenvalue $\lambda_i^{\Tilde{\mathbf{A}}}$.
By construction, all transformations $\mathbf{W}^{(k)}$ act the same on each subspace, nullifying when assessing the relative norm change amongst pairs of subspaces:
\begin{theorem}
\label{pr:symm_iter} (The relative behavior is fixed.)
Let $\mathbf{T} = \mathbf{W}\otimes \Tilde{\mathbf{A}}$ with $\tilde{\mathbf{A}}\in\mathbb{R}^{n\times n}$ symmetric with eigenvectors $\mathbf{v}_1,\dots,\mathbf{v}_n$ and eigenvalues $\lambda_1^{\Tilde{\mathbf{A}}},\dots,\lambda_n^{\Tilde{\mathbf{A}}}$ and $\mathbf{W}\in\mathbb{R}^{d\times d}$ any matrix. 
Consider the bases $\mathbf{S}_{(i)} = \mathbf{I}_d\otimes \mathbf{v}_i$ and $\mathbf{S}_{(j)} = \mathbf{I}_d\otimes \mathbf{v}_j$ for $i,j\in[n]$. Then,    
\begin{equation}
    \frac{\lVert\mathbf{T}\mathbf{S}_{(i)}\rVert_F}{\lVert\mathbf{T}\mathbf{S}_{(j)}\rVert_F} 
    = \frac{|\lambda_i^{\Tilde{\mathbf{A}}}|}{ |\lambda_j^{\Tilde{\mathbf{A}}}|}\, .
\end{equation}
\end{theorem}
We underline the significance of this key property. As the right-hand side is independent of $\mathbf{W}$ with equality, the amplification or reduction of individual subspaces cannot be affected by any feature transformation or the graph signal. The outcome is predefined exclusively by the eigenvalues of the aggregation function. The functions learnable by each operation are inherently restricted. 

\subsection{Implications in the Limit}

Considering the case where graph convolutions are repeatedly applied, we extend our results directly to the iterated case and different transformations $\mathbf{T}^{(k)}$ at each layer $k$, yielding the dominance of fixed subspaces:

\begin{proposition} (Fixed subspaces dominate.)
Let $\mathbf{T}^{(k)} = \mathbf{W}^{(k)}\otimes \Tilde{\mathbf{A}}$ with $\tilde{\mathbf{A}}\in\mathbb{R}^{n\times n}$ symmetric with eigenvectors $\mathbf{v}_1,\dots,\mathbf{v}_n$ and eigenvalues $\lambda_1^{\Tilde{\mathbf{A}}},\dots,\lambda_n^{\Tilde{\mathbf{A}}}$ and $\mathbf{W}^{(k)}\in\mathbb{R}^{d\times d}$ any matrix. 
Consider the bases $\mathbf{S}_{(i)} = \mathbf{I}_d\otimes \mathbf{v}_i$ for $i\in[n]$. Then,   
    \begin{equation*}\lim_{l\to\infty}\frac{\lVert\mathbf{T}^{(l)}\dots\mathbf{T}^{(1)}\mathbf{S}_{(i)}\rVert_F}{\max_{p}\lVert\mathbf{T}^{(l)}\dots\mathbf{T}^{(1)}\mathbf{S}_{(p)}\rVert_F} = \lim_{l\to\infty}\frac{|\lambda^{\Tilde{\mathbf{A}}}_i|^l\cdot\lVert\mathbf{S}_{(i)}\rVert_F}{\max_p|\lambda^{\Tilde{\mathbf{A}}}_p|^l\cdot\lVert\mathbf{S}_{(p)}\rVert_F} = \begin{cases}
    1, & \mathrm{if} |\lambda^{\Tilde{\mathbf{A}}}_i|=|\lambda^{\Tilde{\mathbf{A}}}_{1}|\\ 0, & \mathrm{otherwise}
    \end{cases}\end{equation*}
with convergence rate $\frac{|\lambda^{\Tilde{\mathbf{A}}}_i|}{|\lambda^{\Tilde{\mathbf{A}}}_1|}$.
\end{proposition}

As depth increases, only signals corresponding to the largest eigenvalue of the aggregation function significantly influence representations, while those in other subspaces become negligible.
The feature transformations only affect the scale of all subspaces by a shared scalar.
As previous studies did not have insights into this dominance of a subspace, they claimed potential over-separation of node representations when this scalar diverges to infinity~\cite{oono2019graph,cai2020note,zhou2021dirichlet}. 
Contrarily, considering the normalized representations $\frac{\mathbf{X}^{(l)}}{||\mathbf{X}^{(l)}||_F}$ seems to be necessary given our insights. 
We now assume that $|\lambda_2^{\Tilde{\mathbf{A}}}| < |\lambda_1^{\Tilde{\mathbf{A}}}|$ is strict, so that a single subspace dominates. 

While the bases in all subspaces are the same, the distribution of the signals might be heavily skewed. For example, the signal in $\mathcal{Q}_1$ might align with the nullspace of $\mathbf{W}^{(l)}\dots\mathbf{W}^{(1)}$, while the signal in other subspaces might not. 
However, the probability of $\mathbf{W}^{(l)}\dots\mathbf{W}^{(1)}$ satisfying this property converges to zero at an exponential rate with respect to the Lebesgue measure.
This allows us to formally show the convergence of the node representations to $\mathcal{Q}_1$ in probability: 
\begin{theorem} (Representations converge in probability to a fixed set.)
\label{pr:symm_subspace}
Let $\mathbf{X}^{(k+1)} = \Tilde{\mathbf{A}}\mathbf{X}^{(k)}\mathbf{W}^{(k)}$ with $\tilde{\mathbf{A}}\in\mathbb{R}^{n\times n}$ symmetric with eigenvectors $\mathbf{v}_1,\dots,\mathbf{v}_n$ and $\mathbf{W}^{(k)}\in\mathbb{R}^{d\times d}$ any matrix. Consider the subspaces $\mathcal{Q}_i = \spn(\mathbf{I}_d\otimes \mathbf{v}_i)$ for $i\in[n]$. Then, 
    \begin{equation*}
    \forall\ \epsilon > 0\ldotp\exists\ N \in\mathbb{N}\ldotp\forall\ l\in\mathbb{N} \ \textrm{with}\ l > N\ldotp\exists\ \mathbf{m}^{(l)}\in\mathcal{Q}_1 \colon \mathbb{P}\left(\left\lVert\frac{\mathbf{X}^{(l)}}{\lVert\mathbf{X}^{(l)}\rVert_F} - \mathbf{m}^{(l)}\right\rVert_2 > \epsilon\right)=0\,.\end{equation*}
\end{theorem}
Thus, the probability of avoiding convergence to $\mathcal{Q}_1$ exponentially converges to zero with increased model depth. While the chances always remain non-zero, finding these solutions is particularly challenging for an optimization method based on gradient descent. The heavily constrained solution space directly limits the model's expressivity. 
In addition, any slight change in the input can cause the signal to no longer align with the null space of the feature transformations, making a method with this property meaningless for practical applications.

We emphasize that smoothing has yet to be part of our analysis. Our analysis and proofs are more general, relying solely on the symmetry of the aggregation function $\Tilde{\mathbf{A}}$. 
Smoothing occurs only for specific aggregations for which $\mathcal{Q}_1$ consists of smooth signals.
The symmetrically normalized aggregation function $\mathbf{D}^{-\frac{1}{2}}\mathbf{A}\mathbf{D}^{-\frac{1}{2}}$ has this property, as its dominant eigenvector is $\mathbf{v}_1=\mathbf{D}^{\frac{1}{2}}\mathbf{1}$ and $\lambda_1^{\Tilde{\mathbf{A}}}$ is unique~\citep{von2007tutorial}. 
These novel insights enable us to frame our findings using the well-established Dirichlet energy of the normalized node representations, which corresponds to the Rayleigh quotient $E(\mathbf{X}/||\mathbf{X}||_F)=tr(\mathbf{X}^T\mathbf{\Delta X})/||\mathbf{X}||_F^2$~\cite{cai2020note}. The Dirichlet energy of the normalized state was also proposed by~\citet{di2022graph}. 
The key property needed is the equivalence of the dominating subspace $\mathcal{Q}_1$ and the null space of $\mathbf{I}_d\otimes \mathbf{\Delta}$, allowing us to provide the novel proof for feature transformations in probability:
\begin{proposition} (Over-smoothing happens for all $\mathbf{W}^{(k)}$ in probability.)
\label{pr:symm_dir}
Let $\mathbf{X}^{(k+1)} = \Tilde{\mathbf{A}}\mathbf{X}^{(k)}\mathbf{W}^{(k)}$ with $\tilde{\mathbf{A}} = \mathbf{D}^{-\frac{1}{2}}\mathbf{A}\mathbf{D}^{-\frac{1}{2}}$, $\mathbf{W}^{(k)}\in\mathbb{R}^{d\times d}$, and $E(\mathbf{X}) = tr(\mathbf{X}^T\mathbf{\Delta}\mathbf{X})$ for $\mathbf{\Delta} = \mathbf{I} - \Tilde{\mathbf{A}}$.
Then, 
    \begin{equation}
    \plim_{l\to\infty} E\left(\frac{\mathbf{X}^{(l)}}{\lVert\mathbf{X}^{(l)}\rVert_F}\right) = 0
    \end{equation}
with convergence rate $\left(\frac{\lambda_2^\mathbf{\Tilde{A}}}{\lambda_1^\mathbf{\Tilde{A}}}\right)^2$.
\end{proposition} 

\subsection{Extending the Analysis to Arbitrary Aggregations}
We now drop the assumption of a symmetric graph and allow for arbitrary edge weights between any pair of nodes, a concept recently proposed to account for attraction or repulsion between nodes~\cite{chien2020adaptive,bo2021beyond,yan2021two,bodnar2022neural}. While prior investigations (e.g., ~\citep{oono2019graph,cai2020note,zhou2021dirichlet}) focused on specific types of aggregations that were not easily generalizable, our approach provides a more versatile framework.
We find that Theorem~\ref{pr:symm_iter} can be extended to arbitrary aggregation matrices $\Tilde{\mathbf{A}}$, revealing that the dominating signal for these graph convolutions depends solely on the dominating signal of $\Tilde{\mathbf{A}}$: 

\begin{proposition}
\label{pr:jordan_dom}(informal)
For any matrices $\Tilde{\mathbf{A}}\in\mathbb{R}^{n\times n}$ and $\mathbf{W}\in\mathbb{R}^{d\times d}$, the relative amplification of subspaces only depends on $\Tilde{\mathbf{A}}$ and not on $\mathbf{W}$. 
\end{proposition}
Importantly, invariant bases always exist, as given by the Jordan normal form~\cite{strang2006linear}.
Over-smoothing is prevented by selecting or learning an aggregation matrix $\Tilde{\mathbf{A}}$ with a well-suited spectrum.

\subsection{Stochastic Aggregation Functions}
Our findings can be readily extended to cases where we know about the dominance of a particular signal for the aggregation function.
This applies to all row-stochastic aggregations, also referred to as the weighted mean, where the values in each row of $\Tilde{\mathbf{A}}\in\mathbb{R}^{n\times n}$ are non-negative and sum up to one.
This analysis covers models such as the Graph Attention Network~(GAT)~\citep{velivckovic2017graph,brody2021attentive} with a single attention head.
Edge weights may be dynamically computed based on their adjacent nodes, so we consider the case of time-inhomogeneous aggregation matrices $\Tilde{\mathbf{A}}{(0)} \neq \Tilde{\mathbf{A}}{(1)} \neq \dots \neq \Tilde{\mathbf{A}}{(l)}$. 
Assuming the underlying graph is ergodic, each $\Tilde{\mathbf{A}}^{(k)}$ possesses the same right-eigenvector $\mathbf{p}_1 = \mathbf{1}$ with corresponding eigenvalue $\lambda_1^{\Tilde{\mathbf{A}}^{(k)}} = 1$. All other eigenvalues are strictly less in absolute value as given by the Perron-Frobenius theorem~\cite{perron1907theorie} and stochastic processes~\citep{gallager1997discrete}. 
Given a minimum edge weight bound $\epsilon > 0$, we build on the insight from \citet{wu2023demystifying} that any product $\Pi_{k=0}^\infty \Tilde{\mathbf{A}}^{{(k)}}$ also converges to a matrix with constant rows.
This subspace dominates the representations as given by Proposition~\ref{pr:jordan_dom}, and we show over-smoothing using the Dirichlet energy in probability to exclude the cases where the nullspace of $\mathbf{W}^{(1)}\dots\mathbf{W}^{(l)}$ aligns with the signal in the dominating subspace:
\begin{proposition}
\label{pr:gat_smooth}(GAT and Graph Transformer over-smooth.)
Let $\mathbf{X}^{(k+1)} = \Tilde{\mathbf{A}}^{(k)}\mathbf{X}^{(k)}\mathbf{W}^{(k)}$ where all $\tilde{\mathbf{A}}^{(k)}\in\mathbb{R}^{n\times n}$ row-stochastic, representing the same ergodic graph with minimum non-zero value $>\epsilon$ for some $\epsilon>0$ and $\mathbf{W}_i^{(k)}\in\mathbb{R}^{d\times d}$.
Then, 
    $$\plim_{l\to\infty} E_{\mathrm{rw}}\left(\frac{\mathbf{X}^{(l)}}{\lVert\mathbf{X}^{(l)}\rVert_F}\right) = 0\,.$$
\end{proposition}

\section{Rank Collapse and Over-Correlation}
We have seen that the choice of the aggregation function leads to over-smoothing, which can, therefore, be prevented by selecting an aggregation with a different spectrum. 
However, for any aggregation function $\Tilde{\mathbf{A}}$, the representations are dominated by some low-dimensional subspace $\mathcal{Q}_1 = \spn(\mathbf{I}_d\otimes \mathbf{V}_1)$ based on some matrix $\mathbf{V}_1\in\mathbb{R}^{n\times j}$ and the algebraic multiplicity $j$ of all eigenvalues with the maximal absolute value $|\lambda_1^{\Tilde{\mathbf{A}}}|$:

\begin{theorem}(informal)
    For all $\Tilde{\mathbf{A}}\in\mathbb{R}^{n\times n}$ and for $\mathbf{W}^{(l)}\dots\mathbf{W}^{(1)}$ in probability, in the limit $l\to\infty$, the rank of $\mathbf{X}^{(l)}$ is bounded the joint algebraic multiplicity of all eigenvalues of $\Tilde{\mathbf{A}}$ with magnitude $|\lambda_1|$.    
\end{theorem}

This implies that, regardless of the aggregation function chosen, the representations collapse in probability to a low-rank state bounded by the algebraic multiplicity $j$ of the dominant eigenvalue.
This poses a significant challenge for the expressiveness of deep graph neural networks, as node representations can only contain, at most, $j$ relevant features.
It also leads to a perfect correlation among features if $\mathrm{rank}(\mathbf{V}_i) = 1$, e.g., when the dominant eigenvalue is unique. \citet{jin2022feature} empirically observed this phenomenon and termed it over-correlation. 
This study reveals that rank collapse is the fundamental issue underlying over-smoothing and over-correlation.
These findings also establish a connection between these phenomena and recent theoretical insights on deep Transformer models, where rank collapse was also independently observed~\cite{dong2021attention,noci2022signal}.




\section{Preventing Rank Collapse with a Sum of Kronecker Products}
The core issue leading to rank collapse is the fact that the transformation matrix $\mathbf{T}=\mathbf{W}\otimes \mathbf{A}$ is composed of a single Kronecker product. To counteract rank collapse and consequently over-smoothing and over-correlation, we need to construct a $\mathbf{T}$ that cannot be decomposed into a single Kronecker product.
We highlight that any matrix
\begin{equation}
    \mathbf{T} = (\mathbf{W}_1\otimes \mathbf{A}_1) + (\mathbf{W}_2\otimes \mathbf{A}_2) + \dots + (\mathbf{W}_p\otimes \mathbf{A}_p)
\end{equation}
can be decomposed into a finite sum of Kronecker products (SKP), as concatenation is a special case~\citep{van1993approximation,cao2021sum}. 
As each term can amplify a different signal per feature column, we show that $d$ terms are sufficient to amplify arbitrary signals across $d$ columns: 
\begin{theorem} (informal) A sum of $d$ Kronecker products $\mathbf{T}=\sum_1^d (\mathbf{W}_i\otimes \Tilde{\mathbf{A}}_i)$ can amplify independent signals for at least $d$ different columns.
\label{pr:skp}
\end{theorem}
The crucial advantage of an SKP is that each $\mathbf{W}_i$ can control the degree to which that term contributes to each feature separately for each column. 
The SKP does not add computational complexity, as we would not compute Kronecker products, i.e., $(\mathbf{W}_1\otimes \Tilde{\mathbf{A}}_1) + \dots + (\mathbf{W}_p\otimes \Tilde{\mathbf{A}}_p))\vect(\mathbf{X}) = \vect(\Tilde{\mathbf{A}}_1\mathbf{X}\mathbf{W}_1 + \dots + \Tilde{\mathbf{A}}_p\mathbf{X}\mathbf{W}_p)$.
It is important to note that the concept of an SKP is not tied to a specific model but rather a guiding principle that helps design methods with favorable properties.
This understanding also provides theoretical insights into the success of many existing methods that can be understood as SKPs, including discrete convolutions~\cite{6795724}, residual connections~\citep{he2016deep,bresson2017residual,li2021training}, multi-head graph attention networks~\cite{velivckovic2017graph}, mixing aggregation functions~\cite{rosenbluth2023some} or transforming signals from incoming and outgoing edges differently~\cite{rossi2023edge}. However, Theorem~\ref{pr:skp} relies on each aggregation matrix having a different dominating eigenvector to ensure the amplification of different signals across feature columns.

\section{Empirical Validation}
\label{sec:validation}
\definecolor{color1}{HTML}{404788}
\definecolor{color2}{HTML}{238A8D}
\definecolor{color3}{HTML}{55C667}
\definecolor{color4}{HTML}{FDE725}

\begin{figure}[tb]
     \centering
     \begin{subfigure}[b]{0.32\textwidth}
         \centering
        \def\svgwidth{\textwidth}
         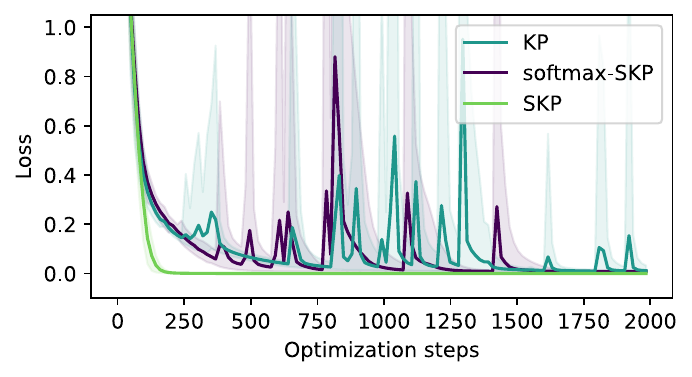
         \caption{Cora}
         \label{fig:constant_zero}
     \end{subfigure}
     \hfill
     \begin{subfigure}[b]{0.32\textwidth}
         \centering
        \def\svgwidth{\textwidth}
         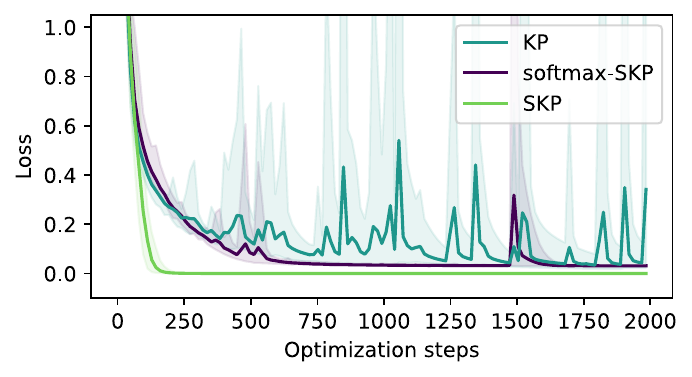
         \caption{Citeseer}
         \label{fig:constant_div}
     \end{subfigure}
     \hfill
     \begin{subfigure}[b]{0.32\textwidth}
         \centering
        \def\svgwidth{\textwidth}
         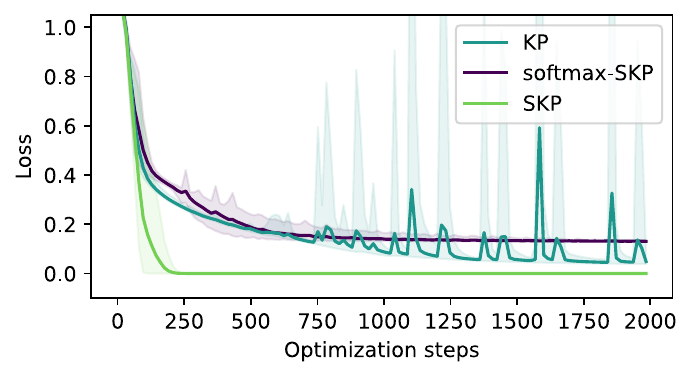
         \caption{Pubmed}
         \label{fig:constant_div}
     \end{subfigure}
          \begin{subfigure}[b]{0.32\textwidth}
         \centering
        \def\svgwidth{\textwidth}
         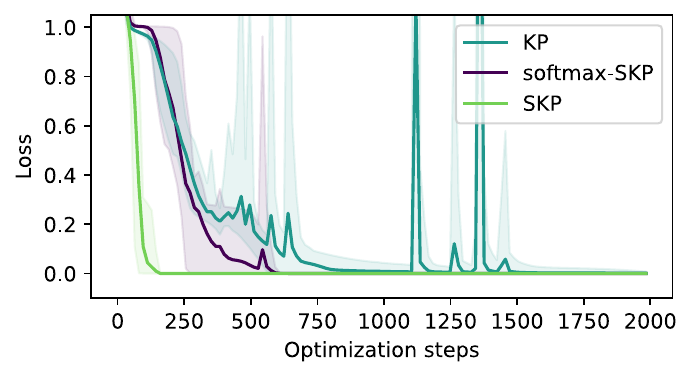
         \caption{Texas}
         \label{fig:constant_zero}
     \end{subfigure}
     \hfill
     \begin{subfigure}[b]{0.32\textwidth}
         \centering
        \def\svgwidth{\textwidth}
         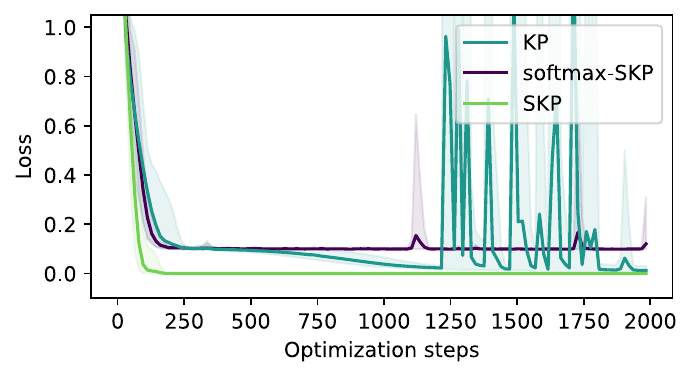
         \caption{Wisconsin}
         \label{fig:constant_div}
     \end{subfigure}
     \hfill
     \begin{subfigure}[b]{0.32\textwidth}
         \centering
        \def\svgwidth{\textwidth}
         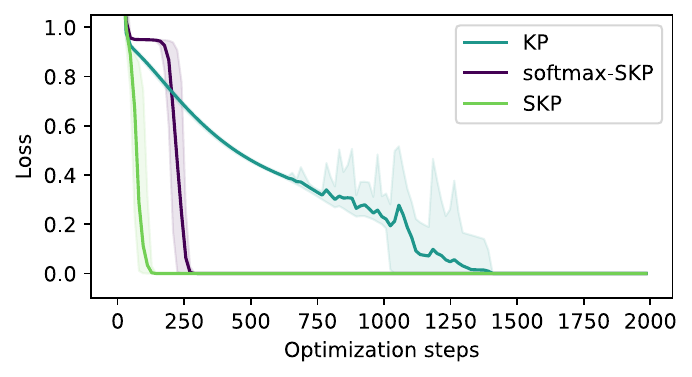
         \caption{Cornell}
         \label{fig:constant_div}
     \end{subfigure}
          \begin{subfigure}[b]{0.32\textwidth}
         \centering
        \def\svgwidth{\textwidth}
         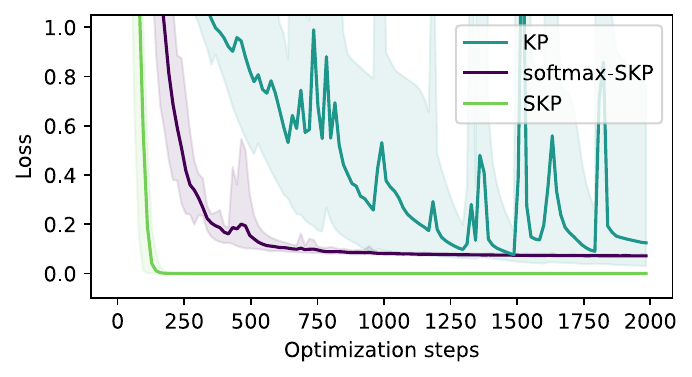
         \caption{Film}
         \label{fig:constant_zero}
     \end{subfigure}
     \hfill
     \begin{subfigure}[b]{0.32\textwidth}
         \centering
        \def\svgwidth{\textwidth}
         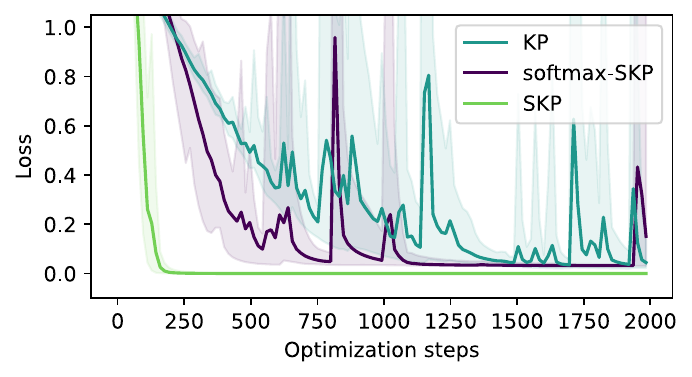
         \caption{Chameleon}
         \label{fig:constant_div}
     \end{subfigure}
     \hfill
     \begin{subfigure}[b]{0.32\textwidth}
         \centering
        \def\svgwidth{\textwidth}
         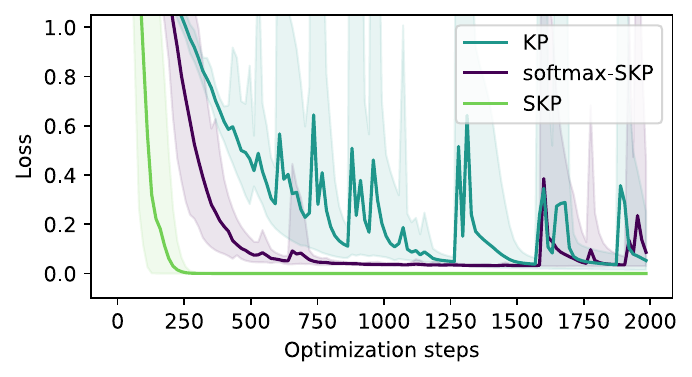
         \caption{Squirrel}
         \label{fig:constant_div}
     \end{subfigure}
        \caption{Mean optimization loss across $10$ runs for various node classification tasks with $8$ layers of message-passing. The shaded region indicates minimum and maximum loss values.}
        \label{fig:results}
\end{figure}

Our theoretical statements show that the current message-passing frameworks suffer from rank collapse, which makes learning complex functions in deep graph neural networks challenging.
To empirically confirm this phenomenon also for the non-linear case of finite depth, we now evaluate the ability of message-passing frameworks to fit the target function of nine common benchmark datasets for node classification. 
\paragraph{Datasets}
Three homophilic citation networks that generally benefit from smoothing are considered, namely the citation datasets Cora~\cite{yang2016revisiting}, Citeseer~\cite{yang2016revisiting}, and Pubmed~\cite{yang2016revisiting}. The other six datasets are heterophilic, of which three are webpage datasets, namely Texas~\cite{pei2019geom}, Wisconsin~\cite{pei2019geom}, Cornell~\cite{pei2019geom}, two are Wikipedia networks, namely Chameleon~\cite{rozemberczki2021multi}, and Squirrel~\cite{rozemberczki2021multi}, and one is the actor co-occurrence network Film~\cite{tang2009social}. For each dataset, we select the largest strongly connected component and optimize the cross-entropy loss for all nodes in order to compare each model's ability to fit the desired function. We do not utilize self-loops.

\paragraph{Considered Methods}
Our theory motivates us to evaluate three different message-passing frameworks: A single Kronecker product, a softmax-activated sum of Kronecker products, and any sum of Kronecker products. 
A single Kronecker product
\begin{align}
    \mathbf{T}^{(k)}_{\mathrm{KP}} &= \mathbf{W}_1^{(k)}\otimes \Tilde{\mathbf{A}}_1^{(k)}
\end{align}
includes most current message-passing frameworks, such as the GCN~\cite{kipf2017semi}, mean aggregation as in GraphSAGE~\cite{hamilton2017inductive}, sum aggregation as in GIN~\cite{xu2018powerful}, negative edge weights as in FAGCN~\citep{bo2021beyond} and many more~\citep{chien2020adaptive,yan2021two,bodnar2022neural}. To cover all of these methods in our experiments, the weights of the available edges in $\Tilde{\mathbf{A}}_1$ are directly learned from the data.
The second framework
\begin{equation}
\mathbf{T}^{(k)}_{\mathrm{softmax-SKP}} = \frac{1}{2}(\mathbf{W}_1^{(k)}\otimes \tau(\Tilde{\mathbf{A}}_1^{(k)})) + \frac{1}{2}(\mathbf{W}_2^{(k)}\otimes \tau(\Tilde{\mathbf{A}}_2^{(k)}))
\end{equation}
extends our proof on time-inhomogeneous softmax-activated aggregation function (Prop.~\ref{pr:gat_smooth}) to the multi-head case, i.e., a sum of Kronecker products for which each aggregation function is activated by a softmax function $\tau$ over each node's neighbors. The weights of all available edges in $\Tilde{\mathbf{A}}_1$ and $\Tilde{\mathbf{A}}_2$ are learned directly before applying $\tau$, which causes the spectrum of both matrices to align and Theorem~\ref{pr:skp} on improved signal amplification to not hold for this model. This represents the multi-head versions of GAT~\cite{velivckovic2017graph}, GATv2~\cite{brody2021attentive}, TransformerConv~\cite{shi2020masked} and other potential attention mechanisms.
Following our theoretical insights on achieving a more powerful message-passing framework, our third approach is based on a sum of Kronecker products
\begin{equation}
    \mathbf{T}^{(k)}_{\mathrm{SKP}} = \mathbf{W}_1^{(k)}\otimes \Tilde{\mathbf{A}}_1^{(k)} + \mathbf{W}_2^{(k)}\otimes \Tilde{\mathbf{A}}_2^{(k)}\, ,
\end{equation} 
where the edge weights are not activated to avoid matching spectra between terms. Based on Theorem~\ref{pr:skp}, this amplifies separate signals per channel and node. 

\paragraph{Experimental Setup}
Our full model starts by encoding the given node features into a $d$-dimensional vector using a linear layer and a ReLU activation. We then apply the described message-passing schemes $\vect(\mathbf{X}^{(k+1)})=\phi(\mathbf{T}^{(k)}\vect(\mathbf{X}^{(k)})$ for $l$ layers, and apply the ReLU activation $\phi$ after each layer to validate that non-linear models also suffer from rank collapse.
After $l$ such layers, the node representations are mapped to class probabilities by an affine transformation and a softmax activation. All model parameters are optimized to minimize the cross-entropy using Adam~\cite{KingmaB14} for $2000$ steps.
Each experiment is repeated for ten random initializations.
Our reproducible implementation is based on PyTorch~\cite{paszke2019pytorch} and PyTorch Geometric~\cite{Fey/Lenssen/2019}\footnote{Our implementation is available at \url{https://github.com/roth-andreas/rank_collapse}}. We refer to Appendix~\ref{sec:appendix_exp} for all details.

\begin{figure}[tb]
     \centering
     \begin{subfigure}[b]{0.32\textwidth}
         \centering
        \def\svgwidth{\textwidth}
         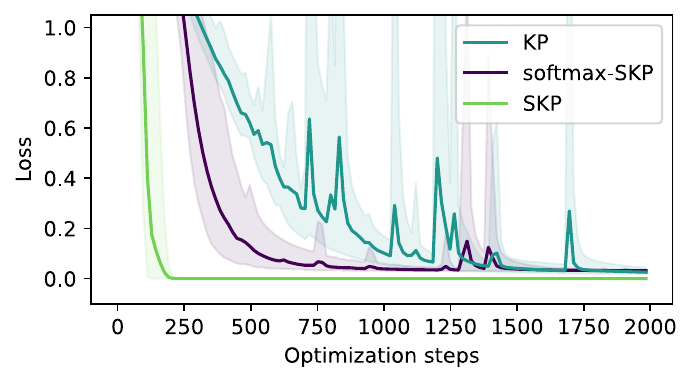
         \caption{$d=16$}
         \label{fig:constant_zero}
     \end{subfigure}
     \hfill
     \begin{subfigure}[b]{0.32\textwidth}
         \centering
        \def\svgwidth{\textwidth}
         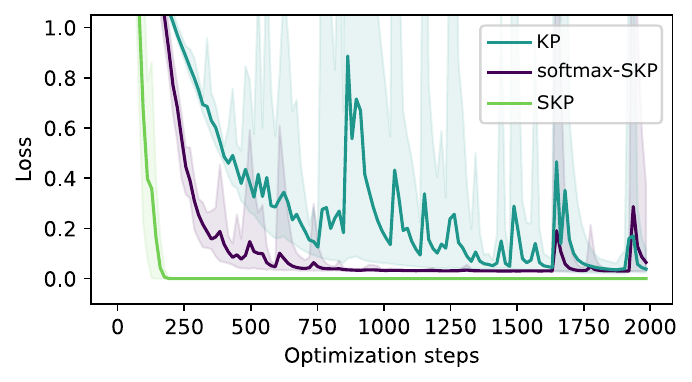
         \caption{$d=64$}
         \label{fig:constant_div}
     \end{subfigure}
     \hfill
     \begin{subfigure}[b]{0.32\textwidth}
         \centering
        \def\svgwidth{\textwidth}
         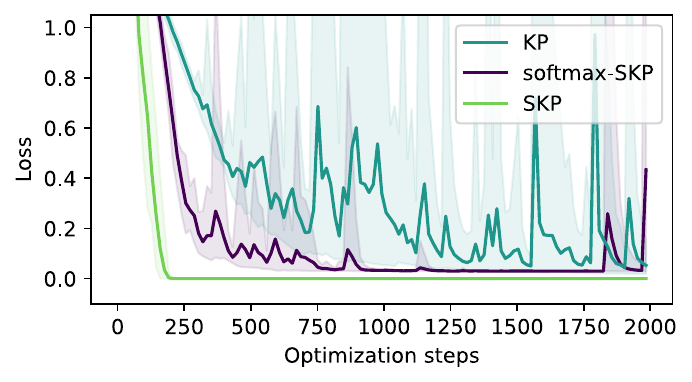
         \caption{$d=128$}
         \label{fig:constant_div}
     \end{subfigure}
        \caption{Optimization loss for the Squirrel dataset when varying the feature dimension $d$.}
        \label{fig:parameters}
\end{figure}

\begin{figure}[tb]
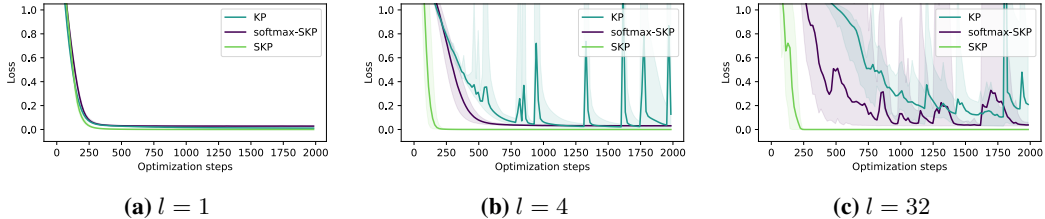

     \centering
     \begin{subfigure}[b]{0.32\textwidth}
         \centering
        \def\svgwidth{\textwidth}
         \input{images/losses_squirrel_1_32.pdf_tex}
         \caption{$l=1$}
         \label{fig:constant_zero}
     \end{subfigure}
     \hfill
     \begin{subfigure}[b]{0.32\textwidth}
         \centering
        \def\svgwidth{\textwidth}
         \input{images/losses_squirrel_4_32.pdf_tex}
         \caption{$l=4$}
         \label{fig:constant_div}
     \end{subfigure}
     \hfill
     \begin{subfigure}[b]{0.32\textwidth}
         \centering
        \def\svgwidth{\textwidth}
         \input{images/losses_squirrel_32_32.pdf_tex}
         \caption{$l=32$}
         \label{fig:constant_div}
     \end{subfigure}
        \caption{Optimization loss for the Squirrel dataset when varying the number of layers $l$.}
        \label{fig:layers}
\end{figure}

\paragraph{Results}
We report the mean, minimal, and maximal achieved loss values over ten runs for $l=8$ layers, a feature dimension of $d=32$, and all considered datasets in Figure~\ref{fig:results}. 
The SKP fits all datasets almost perfectly within less than $250$ steps and decreases almost monotonically. 
Contrarily, KP and softmax-SKP require significantly more steps to converge and remain visibly unstable for most datasets. We ascribe this to both models' inability to amplify different signals per node and channel. 
In most cases, softmax-SKP also converges to a higher loss than the SKP.
To better understand the reasons behind these shortcomings, we provide additional experiments. In Figure~\ref{fig:parameters}, we evaluate the benefit of additional parameters by setting the feature dimension to $d \in \{16,64,128\}$. However, the model size does not help to fit the data better, as the only difference we observe is an increased instability during optimization for larger models.
This strengthens our theory of a more fundamental challenge, which the SKP is able to solve.
Differences in optimization for varying model depth of $l\in\{1,4,32\}$ are displayed in Figure~\ref{fig:layers}. Matching our theory, fitting the data becomes more challenging with increased depth for both the SKP and FAGCN, while the SKP is almost unaffected by the model depth.
We provide an additional experiment in Appendix~\ref{sec:synthetic} that confirms further shortcomings for multi-class classification settings using a synthetic dataset.
These results confirm the severity of rank collapse for common tasks while also extending our theoretical insights to the non-linear case with multiple heads.

\section{Conclusion}
\label{sec:discussion}

Our work shows that rank collapse of node representations is the underlying cause of over-smoothing and over-correlation in GNNs. 
We provided the theoretical foundation that a rank collapse of node representations occurs independently of the chosen aggregation function and the learned feature transformations.
To mitigate these fundamental shortcomings of current message-passing frameworks, we propose the sum of Kronecker products (SKP) as a general property that models should exhibit to provably prevent rank collapse.
We empirically confirm this behavior for non-linear GNNs on nine node classification tasks, which currently employed message-passing frameworks struggle to fit. The SKP easily fits the data even for $32$ layers. 
Our insights show that future methods should aim to avoid rank collapse instead of dealing with over-smoothing or over-correlation. 
Novel metrics to quantify rank collapse in graph neural networks need to be designed, and these need to consider the normalized feature representation, as the feature magnitude may otherwise overshadow the insights.

\section*{Acknowledgements}
This research has been funded by the Federal Ministry of Education and Research of Germany and the state of North-Rhine Westphalia as part of the Lamarr-Institute for Machine Learning and Artificial Intelligence and by the Federal Ministry of Education and Research of Germany under grant no. 01IS22094E WEST-AI.

\bibliographystyle{unsrtnat}
\bibliography{reference}

\clearpage

\appendix
\section{Mathematical details}
\label{sec:proofs}
In this section, we provide the details of our approach and all of our statements.

\subsection{Basic Operations}
We start by listing the most important properties used throughout our formal proofs.

\paragraph{Kronecker Product.}
The Kronecker product for any two matrices $\mathbf{A}\in\mathbb{R}^{p\times q}$, $\mathbf{B}\in\mathbb{R}^{r\times s}$ is denoted as 
\begin{equation}
\mathbf{A}\otimes\mathbf{B}=\begin{bmatrix}
    a_{11}\mathbf{B} & \dots & a_{1q}\mathbf{B}\\
    \vdots & \ddots & \vdots \\
    a_{p1}\mathbf{B} & \dots & a_{pq}\mathbf{B}
\end{bmatrix}\, .
\end{equation}
The importance of the Kronecker product for our work stems from its powerful properties. We briefly present the most relevant here. First, a vectorized matrix product
\begin{equation}
    \vect(\Tilde{\mathbf{A}}\mathbf{XW}) = (\mathbf{W}^T\otimes \Tilde{\mathbf{A}})\vect(\mathbf{X})
\end{equation}
of any matrices $\Tilde{\mathbf{A}},\mathbf{X},\mathbf{W}$ with matching shapes can be written using the Kronecker product. The Kronecker product of two orthogonal matrices results in an orthogonal matrix, allowing us to rewrite any vector as a linear combination
\begin{equation}
    \vect(\mathbf{X}) = (\mathbf{U}\otimes \mathbf{V})\mathbf{c}
\end{equation}
using the singular value decomposition $\mathbf{W}=\mathbf{U}\mathbf{\Sigma}\mathbf{N}^T$ and eigendecomposition $\Tilde{\mathbf{A}}=\mathbf{V}\mathbf{\Lambda}\mathbf{V}^T$. The Kronecker product also satisfies the mixed-product property
\begin{equation}
(\mathbf{A}\otimes \mathbf{B})(\mathbf{C}\otimes \mathbf{D}) = (\mathbf{AC})\otimes (\mathbf{BD})\, .    
\end{equation}

\paragraph{Dirichlet Energy} The standard interpretation for the Dirichlet energy is the sum of differences in representations for adjacent nodes. Another interpretation we mainly use is based on the decomposition of the signal into eigenvectors of the graph Laplacian $\mathbf{\Delta}=\mathbf{I}-\Tilde{\mathbf{A}}$ using $\Tilde{\mathbf{A}}=\mathbf{D}^{-\frac{1}{2}}\mathbf{A}\mathbf{D}^{-\frac{1}{2}}$. Utilizing the Kronecker product, only the signals not belonging to eigenvalue $\lambda_1^{\Tilde{\mathbf{A}}}=1$ are summed up and weighted by the corresponding eigenvalue of $\mathbf{\Delta}$: 
\begin{align}
    E(\mathbf{X}) &= tr(\mathbf{X}^T\mathbf{\Delta X}) \\
    &= \vect(\mathbf{X})^T\vect(\mathbf{\Delta X})\, .
\end{align}
We then use the decomposition $vec(\mathbf{X}) = (\mathbf{I}_d\otimes\mathbf{V})\mathbf{c}$ based on the eigenvectors $\mathbf{V}$ of $\Tilde{\mathbf{A}}$ and the identity matrix $\mathbf{I}_d$, leading to
\begin{align}
    E(\mathbf{X}) &= \mathbf{c}^T(\mathbf{I}_d\otimes\mathbf{V})^T(\mathbf{I}_d\otimes \mathbf{\Delta})(\mathbf{I}_d\otimes\mathbf{V})\mathbf{c} \\
    &= \mathbf{c}^T (\mathbf{I}_d\otimes \mathbf{V}^T\mathbf{\Delta}\mathbf{V})\mathbf{c}\, .
\end{align}
As $\mathbf{\Delta} = \mathbf{V}(\mathbf{I}_n-\mathbf{\Lambda})\mathbf{V}$ has the same eigenvectors and shifted eigenvalues as $\Tilde{\mathbf{A}}$, we then write the statement as a sum of coefficients that are weighted by their corresponding eigenvalue of the graph Laplacian:
\begin{align}
    E(\mathbf{X}) &= \mathbf{c}^T (\mathbf{I}_d\otimes (\mathbf{I}_n-\mathbf{\Lambda}^\mathbf{\Delta})\mathbf{c}\\
    &= \sum_{l,r=1}^{n,d} (1-\lambda_{r}^{\Tilde{\mathbf{A}}})c_{l,r}^2\, .
\end{align}

\paragraph{Frobenius Norm}
The squared Frobenius norm has a similar interpretation, the coefficients are just not weighted by eigenvalues:
\begin{equation}
\begin{split}
    ||\mathbf{X}||_F^2 
    &= tr(\mathbf{X}^T\mathbf{X}) \\
    &= \mathbf{c}^T(\mathbf{U}\otimes\mathbf{V})^T(\mathbf{U}\otimes\mathbf{V})\mathbf{c} \\
    &= \sum_{l,r=1}^{n,d}c_{l,r}^2 
    \end{split}
\end{equation}
An important property of the Frobenius norm in conjunction with the Kronecker product is the following:
\begin{equation}
    \lVert\mathbf{A}\otimes\mathbf{B}\rVert_F = \lVert\mathbf{A}\rVert_F\cdot\lVert\mathbf{B}\rVert_F
\end{equation}

\subsection{Proof of Proposition 4.1}
\begin{proposition*} (Node representations vanish.)
    Let $\mathbf{\Tilde{A}}\in\mathbb{R}^{n\times n}$ be symmetric with maximum absolute eigenvalue $|\lambda_1^{\Tilde{\mathbf{A}}}|=1$, $\mathbf{W}\in\mathbb{R}^{d\times d}$ be any matrix with maximum singular value $\sigma_1^{\mathbf{W}}$, and $\phi$ a component-wise non-expansive mapping satisfying $\phi(\mathbf{0}) = \mathbf{0}$. Then,
    \begin{equation}
        ||\phi(\mathbf{\Tilde{A}XW})||_F \leq \sigma_1^{\mathbf{W}}\cdot ||\mathbf{X}||_F\, .
    \end{equation}
\end{proposition*}

\begin{proof} The key property we use is that the non-expansive property of $\phi(\cdot)$ implies the Lipschitz continuity $||\phi(\mathbf{X}) - \phi(\mathbf{Y})||\leq ||\mathbf{X} - \mathbf{Y}||$:
\begin{align}
    ||\phi(\mathbf{\Tilde{A}XW})||_F &= ||\phi(\mathbf{\Tilde{A}XW}) - \phi(\mathbf{0})||_F \\
    &\leq ||\mathbf{\Tilde{A}XW} - \mathbf{0}||_F \\
    &= ||\mathbf{\Tilde{A}XW}||_F\, .
\end{align}
We then use common bounds on the norm of the matrix product for symmetric matrices using the maximum eigenvalue $\lambda_1^{\Tilde{\mathbf{A}}}$ and for an arbitrary matrix based on the maximum singular value $\sigma_1^{\mathbf{W}}$, resulting in
\begin{align}
    ||\mathbf{\Tilde{A}XW}||_F &\leq |\lambda_1^{\Tilde{\mathbf{A}}}|\sigma_1^{\mathbf{W}}||\mathbf{X}||_F \\
    &= \sigma_1^{\mathbf{W}}||\mathbf{X}||_F\, ,
\end{align}
using the assumption $|\lambda_1^{\Tilde{\mathbf{A}}}|=1$.
\end{proof}

\subsubsection{Proof of Lemma 5.1}
\begin{lemma*}
(The subspaces are invariant to any $\mathbf{T}^{(k)}$.)
Let $\mathbf{T} = \mathbf{W}\otimes \Tilde{\mathbf{A}}$ with $\tilde{\mathbf{A}}\in\mathbb{R}^{n\times n}$ symmetric with eigenvectors $\mathbf{v}_1,\dots,\mathbf{v}_n$ and $\mathbf{W}\in\mathbb{R}^{d\times d}$ any matrix. 
Consider the subspaces $\mathcal{Q}_i = \spn(\mathbf{I}_d\otimes \mathbf{v}_i)$ for $i\in[n]$. Then, 
\begin{equation*}
    \forall\, i\in [n]: \mathbf{z} \in \mathcal{Q}_{i}\implies \mathbf{T}\mathbf{z} \in \mathcal{Q}_{i}\,.
\end{equation*}
\end{lemma*}
\begin{proof}
We express $\mathbf{z}$ as a linear combination $(\mathbf{I}_d\otimes \mathbf{v}_i)\mathbf{c} = \mathbf{z} \in \mathcal{Q}_i$ of the given basis vectors.
Then,
\begin{equation}
\begin{split}
(\mathbf{W}\otimes \Tilde{\mathbf{A}})(\mathbf{I}_d\otimes \mathbf{v}_i)\mathbf{c} &= (\mathbf{W}\mathbf{I}_d\otimes\Tilde{\mathbf{A}}\mathbf{v}_i)\mathbf{c}\\
&= (\mathbf{I}_d\otimes\mathbf{v}_i)(\mathbf{W}\otimes \lambda_i^{\tilde{\mathbf{A}}}\mathbf{I}_n)\mathbf{c}\\
&= (\mathbf{I}_d\otimes\mathbf{v}_i)\mathbf{c}^\prime\in\mathcal{Q}_{i}
\end{split}
\end{equation}
using some new coefficients $\mathbf{c}^\prime = (\mathbf{W}\otimes \lambda_i^{\tilde{\mathbf{A}}}\mathbf{I}_n)\mathbf{c}$.
\end{proof}

\subsubsection{Proof of Theorem 5.2}
\begin{theorem*}
(The relative behavior is fixed.)
Let $\mathbf{T} = \mathbf{W}\otimes \Tilde{\mathbf{A}}$ with $\tilde{\mathbf{A}}\in\mathbb{R}^{n\times n}$ symmetric with eigenvectors $\mathbf{v}_1,\dots,\mathbf{v}_n$ and eigenvalues $\lambda_1^{\Tilde{\mathbf{A}}},\dots,\lambda_n^{\Tilde{\mathbf{A}}}$ and $\mathbf{W}\in\mathbb{R}^{d\times d}$ any matrix. 
Consider the bases $\mathbf{S}_{(i)} = \mathbf{I}_d\otimes \mathbf{v}_i$ and $\mathbf{S}_{(j)} = \mathbf{I}_d\otimes \mathbf{v}_j$ for $i,j\in[n]$. Then,    
\begin{equation}
    \frac{\lVert\mathbf{T}\mathbf{S}_{(i)}\rVert_F}{\lVert\mathbf{T}\mathbf{S}_{(j)}\rVert_F} 
    = \frac{|\lambda_i^{\Tilde{\mathbf{A}}}|}{ |\lambda_j^{\Tilde{\mathbf{A}}}|}\, .
\end{equation}
\end{theorem*}
\begin{proof}
\begin{equation}
    \begin{split}
        \frac{\lVert\mathbf{T} \mathbf{S}_{(i)}\rVert}{\lVert\mathbf{T}\mathbf{S}_{(j)}\rVert} &= \frac{\lVert\left(\mathbf{W}\otimes\Tilde{\mathbf{A}}\right)(\mathbf{I}_d\otimes\mathbf{v}_i)\rVert}{\lVert\left(\mathbf{W}\otimes\Tilde{\mathbf{A}}\right)(\mathbf{I}_d\otimes\mathbf{v}_j)\rVert}\\
        &= \frac{\lVert\left(\mathbf{W}\otimes\lambda_i^{\Tilde{\mathbf{A}}}\mathbf{v}_i\right)\rVert}{\lVert\left(\mathbf{W}\otimes\lambda_j^{\Tilde{\mathbf{A}}}\mathbf{v}_j\right)\rVert}\\
        &= \frac{|\lambda_i^{\Tilde{\mathbf{A}}}|\cdot\lVert\mathbf{W}\rVert\cdot\lVert\mathbf{v}_i\rVert}{|\lambda_j^{\Tilde{\mathbf{A}}}|\cdot\lVert\mathbf{W}\rVert\cdot\lVert\mathbf{v}_j\rVert}\\
        &= \frac{|\lambda_i^{\Tilde{\mathbf{A}}}|}{|\lambda_j^{\Tilde{\mathbf{A}}}|}
    \end{split}
\end{equation}
\end{proof}

\subsubsection{Proof of Proposition 5.3}
\begin{proposition*} (Fixed subspaces dominate.)
Let $\mathbf{T}^{(k)} = \mathbf{W}^{(k)}\otimes \Tilde{\mathbf{A}}$ with $\tilde{\mathbf{A}}\in\mathbb{R}^{n\times n}$ symmetric with eigenvectors $\mathbf{v}_1,\dots,\mathbf{v}_n$ and eigenvalues $\lambda_1^{\Tilde{\mathbf{A}}},\dots,\lambda_n^{\Tilde{\mathbf{A}}}$ and $\mathbf{W}^{(k)}\in\mathbb{R}^{d\times d}$ any matrix. 
Consider the bases $\mathbf{S}_{(i)} = \mathbf{I}_d\otimes \mathbf{v}_i$ for $i\in[n]$. Then,   
    \begin{equation*}\lim_{l\to\infty}\frac{\lVert\mathbf{T}^{(l)}\dots\mathbf{T}^{(1)}\mathbf{S}_{(i)}\rVert_F}{\max_{p}\lVert\mathbf{T}^{(l)}\dots\mathbf{T}^{(1)}\mathbf{S}_{(p)}\rVert_F} = \lim_{l\to\infty}\frac{|\lambda^{\Tilde{\mathbf{A}}}_i|^l\cdot\lVert\mathbf{S}_{(i)}\rVert_F}{\max_p|\lambda^{\Tilde{\mathbf{A}}}_p|^l\cdot\lVert\mathbf{S}_{(p)}\rVert_F} = \begin{cases}
    1, & \mathrm{if} |\lambda^{\Tilde{\mathbf{A}}}_i|=|\lambda^{\Tilde{\mathbf{A}}}_{1}|\\ 0, & \mathrm{otherwise}
    \end{cases}\end{equation*}
with convergence rate $\frac{|\lambda^{\Tilde{\mathbf{A}}}_i|}{|\lambda^{\Tilde{\mathbf{A}}}_1|}$.
\end{proposition*}

\begin{proof}
\begin{equation}
    \begin{split}
        \lim_{l\to\infty}\frac{\lVert\mathbf{T}^{(l)}\dots\mathbf{T}^{(1)}\mathbf{S}_{(i)}\rVert}{\max_{p}\lVert\mathbf{T}^{(l)}\dots\mathbf{T}^{(1)}\mathbf{S}_{(p)}\rVert}
        &= \lim_{l\to\infty}\frac{||(\mathbf{W}^{(l)}\otimes\Tilde{\mathbf{A}})\dots(\mathbf{W}^{(1)}\otimes\Tilde{\mathbf{A}})(\mathbf{I}_d\otimes\mathbf{v}_i)||}{\max_p||(\mathbf{W}^{(l)}\otimes\Tilde{\mathbf{A}})\dots(\mathbf{W}^{(1)}\otimes\Tilde{\mathbf{A}})(\mathbf{I}_d\otimes\mathbf{v}_p)||}\\
        &= \lim_{l\to\infty}\frac{||(\mathbf{W}^{(l)}\dots\mathbf{W}^{(1)}\otimes\left(\lambda_i^{\Tilde{\mathbf{A}}}\right)^l\mathbf{v}_i)||}{\max_p||(\mathbf{W}^{(l)}\dots\mathbf{W}^{(1)}\otimes\left(\lambda_p^{\Tilde{\mathbf{A}}}\right)^l\mathbf{v}_p)||}\\
        &= \lim_{l\to\infty}\frac{|\lambda_i^{\Tilde{\mathbf{A}}}|^l\cdot||\mathbf{v}_i||}{\max_p||\lambda_p^{\Tilde{\mathbf{A}}}|^l|\cdot||\mathbf{v}_p||}\\
        &= \begin{cases}
    1, & \mathrm{if} |\lambda^{\Tilde{\mathbf{A}}}_i|=|\lambda^{\Tilde{\mathbf{A}}}_{1}|\\ 0, & \mathrm{otherwise}
    \end{cases}
    \end{split}
\end{equation}
\end{proof}

\subsubsection{Proof of Theorem 5.4}
\begin{theorem*} (Representations converge in probability to a fixed set.)
Let $\mathbf{X}^{(k+1)} = \Tilde{\mathbf{A}}\mathbf{X}^{(k)}\mathbf{W}^{(k)}$ with $\tilde{\mathbf{A}}\in\mathbb{R}^{n\times n}$ symmetric with eigenvectors $\mathbf{v}_1,\dots,\mathbf{v}_n$ and $\mathbf{W}^{(k)}\in\mathbb{R}^{d\times d}$ any matrix. Consider the subspaces $\mathcal{Q}_i = \spn(\mathbf{I}_d\otimes \mathbf{v}_i)$ for $i\in[n]$. Then, 
    \begin{equation*}
    \forall\ \epsilon > 0\ldotp\exists\ N \in\mathbb{N}\ldotp\forall\ l\in\mathbb{N} \ \textrm{with}\ l > N\ldotp\exists\ \mathbf{m}^{(l)}\in\mathcal{Q}_1 \colon \mathbb{P}\left(\left\lVert\frac{\mathbf{X}^{(l)}}{\lVert\mathbf{X}^{(l)}\rVert_F} - \mathbf{m}^{(l)}\right\rVert_2 > \epsilon\right)=0\,.\end{equation*}
\end{theorem*}
\begin{proof}
Let $\mathbf{\Lambda}$ be the matrix of eigenvalues of $\Tilde{\mathbf{A}}$ and the singular value composition of $\mathbf{W}^{(l)}\dots\mathbf{W}^{(1)}=\mathbf{U}\mathbf{\Sigma}\mathbf{N}^T$. We decompose the initial state $\mathbf{X}^{(0)} = (\mathbf{N}\otimes \mathbf{V})\mathbf{c}$ with $\mathbf{c} = (\mathbf{N}\otimes \mathbf{V})^T\mathbf{X}^{(0)}$. We emphasize that the coefficients, therefore, depend on both the initial state $\mathbf{X}^{(0)}$ and the feature transformations $\mathbf{W}^{(l)}\dots\mathbf{W}^{(1)}$. Then,
\begin{equation}
    \begin{split}
        \left\lVert\frac{\mathbf{X}^{(l)}}{||\mathbf{X}^{(l)}||}-\mathbf{m}^{(l)}\right\rVert &= \left\lVert\frac{(\mathbf{W}^{(l)}\otimes\Tilde{\mathbf{A}})\dots(\mathbf{W}^{(1)}\otimes\Tilde{\mathbf{A}})(\mathbf{N}\otimes\mathbf{V})\mathbf{c}}{||(\mathbf{W}^{(l)}\otimes\Tilde{\mathbf{A}})\dots(\mathbf{W}^{(1)}\otimes\Tilde{\mathbf{A}})(\mathbf{N}\otimes\mathbf{V})\mathbf{c}||}-\mathbf{m}^{(l)}\right\rVert\\
        &= \left\lVert\frac{(\mathbf{W}^{(l)}\dots\mathbf{W}^{(1)}\otimes\mathbf{V})(\mathbf{N}\otimes\mathbf{\Lambda}^l)\mathbf{c}}{||(\mathbf{W}^{(l)}\dots\mathbf{W}^{(1)}\otimes\mathbf{V})(\mathbf{N}\otimes\mathbf{\Lambda}^l)\mathbf{c}||}-\mathbf{m}^{(l)}\right\rVert\\
        &= \left\lVert\frac{\left(\mathbf{W}^{(l)}\dots\mathbf{W}^{(1)}\otimes \mathbf{V}\right)\left(\mathbf{N}\otimes \frac{\mathbf{\Lambda}^l}{\lambda_1^{\Tilde{\mathbf{A}}}}\right)\mathbf{c}}{||\left(\mathbf{W}^{(l)}\dots\mathbf{W}^{(1)}\otimes \mathbf{V}\right)\left(\mathbf{N}\otimes \frac{\mathbf{\Lambda}^l}{\lambda_1^{\Tilde{\mathbf{A}}^l}}\right)\mathbf{c}||}-\mathbf{m}^{(l)}\right\rVert\\
\end{split}
\end{equation}
We simplify the notation and set $\mathbf{P}^{(l)} = \mathbf{W}^{(l)}\dots\mathbf{W}^{(1)}$. We now choose $\mathbf{m}^{(l)} = \frac{\left(\mathbf{P}^{(l)}\otimes \mathbf{V}\right)\left(\mathbf{N}\otimes \frac{\mathbf{\Lambda}^l}{\lambda_1^{\Tilde{\mathbf{A}}^l}}\right)\mathbf{c}^\prime}{||\left(\mathbf{P}^{(l)}\otimes \mathbf{V}\right)\left(\mathbf{N}\otimes \frac{\mathbf{\Lambda}^l}{\lambda_1^{\Tilde{\mathbf{A}}^l}}\right)\mathbf{c}||}$ by only replacing the $\mathbf{c}$ in the numerator by $\mathbf{c}^\prime$ containing the same values for coefficients corresponding to $\mathcal{Q}_1$ and zeros for other subspaces, i.e., $c^\prime_{1,r}=\mathbf{c}_{1,r}\,\forall r$ and $c^\prime_{i,r} = 0\,\forall i>1,\forall r$. Continuing,
\begin{equation}
    \begin{split}
        &= \left\lVert\frac{\left(\frac{\mathbf{P}^{(l)}}{\sigma_1^{\mathbf{P}^{(l)}}}\otimes \mathbf{V}\right)}{||\left(\frac{\mathbf{P}^{(l)}}{\sigma_1^{\mathbf{P}^{(l)}}}\otimes \mathbf{V}\right)\left(\mathbf{N}\otimes \frac{\mathbf{\Lambda}^l}{\lambda_1^{\Tilde{\mathbf{A}}^l}}\right)\mathbf{c}||}\left(\left(\mathbf{N}\otimes \frac{\mathbf{\Lambda}^l}{\lambda_1^{\Tilde{\mathbf{A}}^l}}\right)\mathbf{c}-\mathbf{c}^\prime\right)\right\rVert\\
        &\leq \frac{\left\lVert\frac{\mathbf{P}^{(l)}}{\sigma_1^{\mathbf{P}^{(l)}}}\otimes \mathbf{V}\right\rVert}{\left\lVert\left(\frac{\mathbf{P}^{(l)}}{\sigma_1^{\mathbf{P}^{(l)}}}\otimes \mathbf{V}\right)\left(\mathbf{N}_d\otimes \frac{\mathbf{\Lambda}^l}{\lambda_1^{\Tilde{\mathbf{A}}^l}}\right)\mathbf{c}\right\rVert}\left\lVert\left(\mathbf{N}_d\otimes \frac{\mathbf{\Lambda}^l}{\lambda_1^{\Tilde{\mathbf{A}}^l}}\right)\mathbf{c}-\mathbf{c}^\prime\right\rVert \\
        &\leq \frac{\sqrt{nd}}{c_{1,1}}\left(\frac{\lambda_2}{\lambda_1}\right)^l \max_{r,i} c_{r,i}\sqrt{(n-1)\cdot d} < \epsilon \\
        & \iff \frac{\sqrt{nd}}{\epsilon}\left(\frac{\lambda_2}{\lambda_1}\right)^l \sqrt{(n-1)\cdot d} < \frac{c_{1,1}}{\max_{r,i} c_{r,i}}
    \end{split}
\end{equation}
The left-hand side converges to zero for $l\to\infty$, so the constraint on the coefficients becomes weaker with increased depth.
As the volume of the critical solution space around solutions satisfying $c_{1,1} = 0$ reduces to zero, our state converges in probability. 

We now provide the details for the upper bounds in the last step separately for each term:
\begin{equation}
\begin{split}
    \left\lVert\frac{\mathbf{P}^{(l)}}{\sigma_1^{\mathbf{P}^{(l)}}}\otimes \mathbf{V}\right\rVert 
    &= \sqrt{tr\left(\left(\frac{\mathbf{P}^{(l)}}{\sigma_1^{\mathbf{P}^{(l)}}}\otimes \mathbf{V}\right)^T\left(\frac{\mathbf{W}^{(l)}\dots\mathbf{W}^{(1)}}{\sigma_1^{\mathbf{P}^{(l)}}}\otimes \mathbf{V}\right)\right)} \\
    &= \sqrt{tr\left(\frac{\left(\mathbf{P}^{(l)}\right)^T\mathbf{P}^{(l)}}{\left(\sigma_1^{\mathbf{P}^{(l)}}\right)^2}\otimes \mathbf{I}_n\right)} \\
    &= \sqrt{\sum_{l,r}\frac{\sigma_{l}^2}{\sigma_1^2}}
    \leq \sqrt{nd}
\end{split}
\end{equation}

Next, we provide a lower bound on the denominator:

\begin{equation}
\begin{split}
    \left\lVert\left(\frac{\mathbf{P}^{(l)}}{\sigma_1^{\mathbf{P}^{(l)}}}\otimes \mathbf{V}\right)\left(\mathbf{N}_d\otimes \left(\frac{\mathbf{\Lambda}}{\lambda_1^{\Tilde{\mathbf{A}}}}\right)^l\right)\mathbf{c}\right\rVert
    &=\sqrt{\left(\left(\frac{\mathbf{P}^{(l)}}{\sigma_1^{\mathbf{P}^{(l)}}}\otimes \mathbf{V}\right)\left(\mathbf{N}_d\otimes \frac{\mathbf{\Lambda}^l}{\lambda_1^{\Tilde{\mathbf{A}}^l}}\right)\mathbf{c}\right)^T\left(\frac{\mathbf{P}^{(l)}}{\sigma_1^{\mathbf{P}^{(l)}}}\otimes \mathbf{V}\right)\left(\mathbf{N}_d\otimes \left(\frac{\mathbf{\Lambda}}{\lambda_1^{\Tilde{\mathbf{A}}}}\right)^l\right)\mathbf{c}} \\
    &= \sqrt{\mathbf{c}^T\left(\mathbf{N}\frac{\left(\mathbf{P}^{(l)}\right)^T\mathbf{P}^{(l)}}{\left(\sigma_1^{\mathbf{P}^{(l)}}\right)^2}\mathbf{N}\otimes \left(\frac{\mathbf{\Lambda}}{\lambda_1^{\Tilde{\mathbf{A}}}}\right)^{2l}\right)\mathbf{c}} \\
    &= \sqrt{\sum_{l,r}\left(\frac{\sigma_{l}^{\mathbf{P}^{(l)}}}{\sigma_1^{\mathbf{P}^{(l)}}}\right)^2\mathbf{c}_{l,r}^2 \left(\frac{\lambda_r}{\lambda_1}\right)^{2l}} \\
    &\geq c_{1,1}
\end{split}
\end{equation}

\begin{equation}
    \begin{split}
        \left\lVert\left(\mathbf{N}_d\otimes \frac{\mathbf{\Lambda}^{\mathbf{A}^l}}{\lambda_1^\mathbf{A}}\right)\mathbf{c}-\mathbf{c}^\prime\right\rVert 
        &= \sqrt{\sum_{r,i} \left(\left(\frac{\lambda_i}{\lambda_1}\right)^lc_{r,i} - \left(\frac{\lambda_i}{\lambda_1}\right)^lc_{r,i}^\prime\right)^2} \\
        &\leq \sqrt{\sum_{r=1,i=2}^{d,n} \left(\left(\frac{\lambda_i}{\lambda_1}\right)^lc_{r,i}\right)^2} \\
        &\leq \left(\frac{\lambda_2}{\lambda_1}\right)^l \sqrt{\sum_{r=1,i=2}^{n,d} c_{r,i}^2} \\
        &\leq \left(\frac{\lambda_2}{\lambda_1}\right)^l \sqrt{(n-1)\cdot d\cdot \max_{r,i} c_{r,i}^2} = \left(\frac{\lambda_2}{\lambda_1}\right)^l \max_{r,i} c_{r,i}\sqrt{(n-1)\cdot d}
    \end{split}
\end{equation}

\end{proof}

\subsubsection{Proof of Proposition 5.5}
\begin{proposition*} (Over-smoothing happens for all $\mathbf{W}^{(k)}$ in probability.)
Let $\mathbf{X}^{(k+1)} = \Tilde{\mathbf{A}}\mathbf{X}^{(k)}\mathbf{W}^{(k)}$ with $\tilde{\mathbf{A}} = \mathbf{D}^{-\frac{1}{2}}\mathbf{A}\mathbf{D}^{-\frac{1}{2}}$, $\mathbf{W}^{(k)}\in\mathbb{R}^{d\times d}$, and $E(\mathbf{X}) = tr(\mathbf{X}^T\mathbf{\Delta}\mathbf{X})$ for $\mathbf{\Delta} = \mathbf{I} - \Tilde{\mathbf{A}}$.
Then, 
    \begin{equation}
    \plim_{l\to\infty} E\left(\frac{\mathbf{X}^{(l)}}{\lVert\mathbf{X}^{(l)}\rVert_F}\right) = 0
    \end{equation}
with convergence rate $\left(\frac{\lambda_2^\mathbf{\Tilde{A}}}{\lambda_1^\mathbf{\Tilde{A}}}\right)^2$.
\end{proposition*} 

\begin{proof}
We again use the eigendecomposition $\mathbf{\Tilde{A}} = \mathbf{V}\mathbf{\Lambda}\mathbf{V}^T$, the singular value decomposition $\mathbf{W}^{(1)}\dots\mathbf{W}^{(l)} = \mathbf{P}=\mathbf{U}\mathbf{\Sigma}\mathbf{N}^T$ and the coefficients $\mathbf{c} = (\mathbf{U}\otimes\mathbf{V})^T\vect(\mathbf{X}^{(0)})$ of our initial state. With this, we can simplify the Dirichlet energy
\begin{equation}
    \begin{split}
        E\left(\frac{\mathbf{X}^{(l)}}{||\mathbf{X}^{(l)}||}\right) 
        &= \frac{\mathrm{tr}(\mathbf{X}^{(l)}\mathbf{\Delta}\mathbf{X}^{(l)})}{||\mathbf{X}^{(l)}||^2}\\
        &=\frac{\mathbf{c}^T(\mathbf{U}^T\otimes \mathbf{V}^T)(\mathbf{W}^{(1)}\otimes \mathbf{\Tilde{A}}^T)\dots(\mathbf{W}^{(l)}\otimes \mathbf{\Tilde{A}})(\mathbf{I}_d\otimes\mathbf{\Tilde{\Delta}})(\mathbf{W}^{(l)^T}\otimes \mathbf{\Tilde{A}})(\mathbf{W}^{(1)^T}\otimes \mathbf{\Tilde{A}})(\mathbf{U}\otimes \mathbf{V})\mathbf{c}}{||\mathbf{X}^{(l)}||^2}\\
        &= \frac{\mathbf{c}^T(\mathbf{U}^T\mathbf{W}^{(1)}\dots \mathbf{W}^{(l)}\mathbf{W}^{(l)^T}\dots \mathbf{W}^{(1)^T}\mathbf{U}\otimes \mathbf{V}^T\mathbf{\Tilde{A}}^{l}\mathbf{\Tilde{\Delta}} \mathbf{\Tilde{A}}^l\mathbf{V})\mathbf{c}}{||\mathbf{X}^{(l)}||^2}\\
        &= \frac{\mathbf{c}^T(\mathbf{\Sigma}^2\otimes \mathbf{\Lambda}^l \mathbf{V}^T\mathbf{\Tilde{\Delta}} \mathbf{V\Lambda}^l)\mathbf{c}}{\mathbf{c}^T(\mathbf{\Sigma}^2\otimes \mathbf{\Lambda}^l \mathbf{V}^T\mathbf{V\Lambda}^l)\mathbf{c}}\\
        &= \frac{\mathbf{c}^T(\frac{\mathbf{\Sigma}^2}{\sigma_1^2} \otimes \frac{\mathbf{\Lambda}^l(\mathbf{I}_d - \mathbf{\Lambda})\mathbf{\Lambda}^l}{\lambda_1^\mathbf{\Tilde{A}}\lambda_1^\mathbf{\Tilde{A}}})\mathbf{c}}{\mathbf{c}^T(\frac{\mathbf{\Sigma}^2}{\sigma_1^2}\otimes \frac{\mathbf{\Lambda}^l\mathbf{\Lambda}^l}{\lambda_1^\mathbf{\Tilde{A}}\lambda_1^\mathbf{\Tilde{A}}})\mathbf{c}}\\
        &= \frac{\sum_{p=1,r=1}^{N,d} c_{p,r}^2\frac{\sigma_r^2}{\sigma_1^2}\frac{\lambda_p^{2l}(1-\lambda_p)}{\lambda_1^{2l}})}{\sum_{p=1,r=1}^{N,d} c_{p,r}^2\frac{\sigma_r^2}{\sigma_1^2}\frac{\lambda_p^{2l}}{\lambda_1^{2l}})}\, .
    \end{split}
\end{equation}
For convergence, we now need the last term to be smaller than any given $\epsilon>0$.
Using the property that $(1-\lambda_1^{\Tilde{\mathbf{A}}})=0$, we again only require a constraint on the coefficients 
\begin{equation}
    \frac{c_{1,1}^2}{\max_{p,r} c^2_{p,r}} > \frac{2\cdot\lambda_2^{2l}}{\lambda_1^{2l}\epsilon}\, .
\end{equation}
Which similarly converges in probability as the right-hand side converges to zero.
\end{proof}

\subsection{Proof of Proposition 5.6}

\begin{proposition*}
(Signal amplification only depends on $\Tilde{\mathbf{A}}$.) 
Let two bases be $\mathbf{S}_{(i)}=(\mathbf{I}_d\otimes \mathbf{P}_{(i)})\in\mathbb{R}^{nd\times qd}$ and $\mathbf{S}_{(j)}=(\mathbf{I}_d\otimes \mathbf{P}_{(j)})\in\mathbb{R}^{nd\times qd}$ for any
$\mathbf{P}_{(i)}\in\mathbb{R}^{n\times q},\mathbf{P}_{(j)}\in\mathbb{R}^{n\times r}$ with $q,r\leq n$. Further let $\mathbf{T}=\mathbf{\mathbf{W}}\otimes \Tilde{\mathbf{A}}\in\mathbb{R}^{nd\times nd}$ consisting of any $\mathbf{W}\in\mathbb{R}^{d\times d}$ and any $\Tilde{\mathbf{A}}\in\mathbb{R}^{n\times n}$. 
Then, 
        \begin{equation*}
        \frac{||\mathbf{T}\mathbf{S}_{(i)}||_F}{||\mathbf{T}\mathbf{S}_{(j)}||_F}
        = \frac{||\Tilde{\mathbf{A}}\mathbf{P}_{(i)}||_F}{||\Tilde{\mathbf{A}}\mathbf{P}_{(j)}||_F}\,.\
    \end{equation*}
\end{proposition*}

\begin{proof}
We again use the property $||\mathbf{A}\otimes \mathbf{B}||_F=||\mathbf{A}||_F\cdot||\mathbf{B}||_F$ and basic properties of the Kronecker product: 
\begin{align}
     \frac{||\mathbf{T}\mathbf{S}_{(i)}||_F}{||\mathbf{T}\mathbf{S}_{(j)}||_F}
     &=  \frac{||(\mathbf{W}\otimes \Tilde{\mathbf{A}})(\mathbf{I}_d\otimes\mathbf{P}_{(i)})||_F}{||(\mathbf{W}\otimes \Tilde{\mathbf{A}})(\mathbf{I}_d\otimes\mathbf{P}_{(j)})||_F} \\
     &= \frac{||\mathbf{W}||_F\cdot ||\Tilde{\mathbf{A}}\mathbf{P}_{(i)}||_F}{||\mathbf{W}||\cdot||\Tilde{\mathbf{A}}\mathbf{P}_{(j)}||_F} \\
     &= \frac{||\Tilde{\mathbf{A}}\mathbf{P}_{(i)}||_F}{||\Tilde{\mathbf{A}}\mathbf{P}_{(j)}||_F}
\end{align}

\end{proof}

\subsubsection{Proof of Proposition 5.7}
\begin{proposition}
(GAT and Graph Transformer over-smooth.)
Let $\mathbf{X}^{(k+1)} = \Tilde{\mathbf{A}}^{(k)}\mathbf{X}^{(k)}\mathbf{W}^{(k)}$ where all $\tilde{\mathbf{A}}^{(k)}\in\mathbb{R}^{n\times n}$ row-stochastic, representing the same ergodic graph, and each non-zero entry $(\Tilde{\mathbf{A}}^{(k)})_{pq} \geq \epsilon$ for some $\epsilon > 0$ and all $i\in[h]$, $k\in\mathbb{N}$, $p,q\in[n]$. We further allow all $\mathbf{W}_i^{(k)}\in\mathbb{R}^{d\times d}$ to be any matrices. 
We consider $E(\mathbf{X}) = tr(\mathbf{X}^T\mathbf{\Delta}_{rw}\mathbf{X})$ using the random walk Laplacian $\mathbf{\Delta}_{rw} = \mathbf{I} - \mathbf{D}^{-1}\mathbf{A}$.
Then, 
    $$\plim_{l\to\infty} E\left(\frac{\mathbf{X}^{(l)}}{\lVert\mathbf{X}^{(l)}\rVert_F}\right) = 0\,.$$
\end{proposition}
\begin{proof} 
The proof is mostly analogous to our proof of Theorem~\ref{pr:symm_dir} but relies on the fact that the product of time-inhomogenous row-stochastic matrices $\lim_{l\to\infty}\Pi_{i=0}^l \Tilde{\mathbf{A}}^{(i)} = \mathbf{1}(\mathbf{y})^T$ converges to constant columns for some $\mathbf{y}\in\mathbb{R}^{n}$ as established by \citet{wu2023demystifying}. The intuition behind the statement builds on the ergodicity and minimum edge weight $\epsilon$, that combined result in each pairwise edge strength being larger than $\epsilon^s$ after a finite number of steps $s$. The product of two such matrices reduces the maximum and increases the minimum, which results in constant states in the limit.

Let $\mathbf{\Delta}_{rw} = \mathbf{V}(\mathbf{I}_n - \Lambda^{\Tilde{\mathbf{A}}})\mathbf{V}^T$. As the first eigenvector $\mathbf{v}_1$ is constant, we can write the limit state $\lim_{l\to\infty}\Pi_{i=0}^l \Tilde{\mathbf{A}}^{(i)} = \mathbf{VQ}^{(l)}\mathbf{V}^T = \mathbf{V}\begin{bmatrix}
    y^{(l)}_1 & \dots & y_n^{(l)} \\
    q^{(l)}_{21} & \dots & q^{(l)}_{2n} \\
    \vdots & \ddots & \vdots\\
    q^{(l)}_{n2} & \dots & q^{(l)}_{nn}
\end{bmatrix}\mathbf{V}^T$ using the same eigenbasis with each $q_{kp}^{(l)}$ converging to zero with some upper bound $Cq^l$ for some $C > 0$ and $q\in[0,1]$, and each $y^{(l)}_i$ converges to $\mathbf{y}^T\mathbf{v}_i$, as given by~\citet{wu2023demystifying}. For notational simplicity we define $\mathbf{R}^{(l)}=\Pi_{i=0}^{(l)} \mathbf{W}^{(i)}$ and its singular value decomposition as $\mathbf{R}^{(l)} = \mathbf{U}\mathbf{\Sigma}\mathbf{N}^T$. Decomposing the initial state $\vect(\mathbf{X}^{(0)}) = (\mathbf{N}\otimes \mathbf{V})\mathbf{c}$ using $\mathbf{V}$ and $\mathbf{N}$, the Dirichlet energy then simplifies to
\begin{align*}
        \lim_{l\to\infty}E\left(\frac{\mathbf{VQ}^{(l)}\mathbf{V}^T\mathbf{X}^{(0)}\mathbf{R}}{||\mathbf{VQ}^{(l)}\mathbf{V}^T\mathbf{X}^{(0)}\mathbf{R}||_F}\right) 
        &= \lim_{l\to\infty}\frac{\mathbf{c}^T(\mathbf{\Sigma}^2\otimes \mathbf{V}^T\mathbf{VQ}^{(l)}\mathbf{V}^T\mathbf{V}(\mathbf{I}_n - \Lambda^{\Tilde{\mathbf{A}}})\mathbf{V}^T \mathbf{VQ}^{(l)}\mathbf{V}^T\mathbf{V})\mathbf{c}}{\mathbf{c}(\mathbf{\Sigma}^2\otimes \mathbf{V}^T\mathbf{VQ}^{(l)}\mathbf{V}^T\mathbf{VQ}^{(l)}\mathbf{V}^T\mathbf{V})\mathbf{c}}\\
        &= \lim_{l\to\infty}\frac{\mathbf{c}^T(\mathbf{\Sigma}^2\otimes \mathbf{Q}^{(l)^T}(\mathbf{I}_n - \Lambda^{\Tilde{\mathbf{A}}}) \mathbf{Q}^{(l)})\mathbf{c}}{\mathbf{c}(\mathbf{\Sigma}^2\otimes \mathbf{Q}^{(l)^T}\mathbf{Q}^{(l)})\mathbf{c}}
\end{align*}

We then use the fact that $\lim_{l\to\infty} \mathbf{Q}^{(l)^T}(\mathbf{I}_n - \Lambda^{\Tilde{\mathbf{A}}}) \mathbf{Q}^{(l)} = \mathbf{0}$ as the columns of $\mathbf{Q}^{(l)}$ and the rows of $\mathbf{Q}^{(l)^T}$ converge to zero and the corresponding eigenvalue of $(\mathbf{I}_n - \Lambda^{\Tilde{\mathbf{A}}})$ being also zero, leaving all entries to converge to zero. We then simply need to factor out the growth of $R$, so that this convergence holds:

\begin{align}
        \frac{\mathbf{c}^T(\mathbf{\Sigma}^2\otimes \mathbf{Q}^{(l)^T}(\mathbf{I}_n - \Lambda^{\Tilde{\mathbf{A}}}) \mathbf{Q}^{(l)})\mathbf{c}}{\mathbf{c}(\mathbf{\Sigma}^2\otimes \mathbf{Q}^{(l)^T}\mathbf{Q}^{(l)})\mathbf{c}} 
        &= \frac{\mathbf{c}^T(\frac{\mathbf{\Sigma}^2}{\sigma_1^2}\otimes \mathbf{Q}^{(l)^T}(\mathbf{I}_n - \Lambda^{\Tilde{\mathbf{A}}}) \mathbf{Q}^{(l)})\mathbf{c}}{\mathbf{c}(\frac{\mathbf{\Sigma}^2}{|\sigma_1^2|}\otimes \mathbf{Q}^{(l)^T}\mathbf{Q}^{(l)})\mathbf{c}} \\
        &\leq \frac{\sum_{p,r} c_{p,r}^2 \frac{\sigma_r}{\sigma_1} (1-\lambda_p)}{}
\end{align}
To ensure that the last term converges to a value below a given $\epsilon$, we again require 
\begin{equation}
\frac{c_{1,1}}{\max_{p,r} c_{p,r}} > \frac{2(d-1)\cdot C q^l}{\epsilon}\, .
\end{equation}
The right-hand side converges to zero as per assumption, so the overall convergence happens in probability.
\end{proof}

\subsection{Over-correlation}
\subsubsection{Proof of Theorem 6.1}
\begin{theorem*}
    Let $\mathcal{R}_j$ be the space of matrices with rank at most $j$.
    Then, for all $\Tilde{\mathbf{A}}\in\mathbb{R}^{n\times n}$ and for $\mathbf{W}^{(l)}\dots\mathbf{W}^{(1)}$ in probability with respect to the Lebesgue measure, we have
        \begin{equation}
        \forall\ \epsilon > 0\ldotp\exists\ N\in\mathbb{N}\ldotp\forall\ l > N\ldotp L^{(n)}\in\mathcal{R}_j\colon \mathbb{P}\left(||X^{(n)} - L^{(n)}|| > \epsilon\right) = 0\, .
    \end{equation}
\end{theorem*}
\begin{proof}
    We show that $\mathbf{X}^{(l)}$ converges to a sum of a matrix with rank at most $j$ and a matrix that vanishes. We construct subspaces given by the Jordan blocks of the Jordan normal form $\Tilde{\mathbf{A}} = \mathbf{P}\mathbf{J}\mathbf{P}^{-1}$. We merge all subspaces corresponding to an eigenvalue with magnitude $|\lambda_1|$ into $\mathcal{Z}=\spn(\mathbf{I}_d\otimes \mathbf{P}_1)$ with $\mathbf{P}_1\in\mathbb{R}^{n\times j}$ where $j$ is the joint algebraic multiplicity of all eigenvalues with magnitude $|\lambda_1|$. Using these subspaces within Theorem~\ref{pr:symm_subspace}, $\mathbf{X}^{(l)}$ converges in probability to $\mathcal{Z}$. It is then only left to show that $\mathcal{Z}$ corresponds to a subspace of $\mathcal{R}_j$: 
    For every $\mathbf{h}\in \mathcal{Z}$ with $\mathbf{h} = (\mathbf{I}_d\otimes \mathbf{P}_1)\mathbf{c}$, we have
    \begin{equation}
        \vect^{-1}(\mathbf{h}) = \vect^{-1}((\mathbf{I}_d\otimes \mathbf{P}_1) \mathbf{c}) = \mathbf{P}_1 \vect^{-1}(\mathbf{c}^T) \mathbf{I}_d = \mathbf{P}_1 \vect^{-1}(\mathbf{c}^T)
    \end{equation}
    which has rank at most $j$.
\end{proof}

\subsection{Sum of Kronecker products}
\subsection{Proof of Theorem 7.1}
\begin{theorem*} (Any subspace can get amplified.)
    Let $\mathbf{e}_i\in\mathbb{R}^{d}$ be the canonical basis with a single $1$ at position $i$. For any columns $\mathbf{s}_1,\dots,\mathbf{s}_d\in\mathbb{R}^{n}$ and the induced subspace $\mathbf{S} = \begin{bmatrix}
        \mathbf{e}_1 \otimes \mathbf{s}_1 & \dots & \mathbf{e}_d \otimes \mathbf{s}_d
    \end{bmatrix}\in\mathbb{R}^{nd\times d}$, 
    there exists an $\mathbf{T}=\sum_1^d (\mathbf{W}_i\otimes \Tilde{\mathbf{A}}_i)$ such that for all orthogonal bases $\mathbf{S}^\prime = \begin{bmatrix}
        \mathbf{e}_1 \otimes \mathbf{s}^\prime_1 & \dots & \mathbf{e}_d \otimes \mathbf{s}^\prime_d
    \end{bmatrix}\in\mathbb{R}^{nd\times d}$
    \begin{equation}
        \frac{\lVert \mathbf{T}\mathbf{S}\rVert_F}{\lVert \mathbf{T}\mathbf{S}^\prime\rVert_F} > \frac{\lVert\mathbf{S}\rVert_F}{\lVert\mathbf{S}^\prime\rVert_F}\, .
    \end{equation}
\end{theorem*}
\begin{proof}
We choose each $\Tilde{\mathbf{A}}_i=\mathbf{V}^{(i)}\mathbf{\Lambda}\mathbf{V}^{(i)^T}$ to be symmetric with dominant eigenvectors $\mathbf{v}_1^{(i)}=\mathbf{s}_i$ and shared eigenvalues $|\lambda_1| > |\lambda_2^{\Tilde{\mathbf{A}}}| > |\lambda_d^{\Tilde{\mathbf{A}}}| > 0$. Further, $\mathbf{W}_i = \mathrm{diag}(\mathbf{e}_i)$ where the diag operation creates a matrix with the entries of its arguments along the diagonal. Thus, all columns are independent and $\mathbf{T}$ is a block-diagonal matrix  with eigenvectors being $\mathbf{e}_i\otimes\mathbf{v}^{(j)}$ with corresponding eigenvalue $\lambda_j^{\Tilde{\mathbf{A}}}$. The eigenspace corresponding to $\lambda_1^{\Tilde{\mathbf{A}}}$ are all $\mathbf{e}_i\otimes\mathbf{s}_{i}$ for all $i$, i.e., $\spn(\mathbf{S})$. Any orthogonal column $\mathbf{s}^\prime_i = \sum_{k=2}^d \mathbf{v}^{(k)}_2c^{(k)}_2$ can be written as a linear combination of the other eigenvectors. Thus, 

\begin{equation}
    \begin{split}
        \frac{\lVert \mathbf{T}\mathbf{S}\rVert_F}{\lVert \mathbf{T}\mathbf{S}^\prime\rVert_F} &\geq \frac{|\lambda_1|}{|\lambda_2|}\frac{\lVert \mathbf{S} \rVert_F}{\lVert \mathbf{S}^\prime\rVert_F} > \frac{\lVert\mathbf{S}\rVert_F}{\lVert\mathbf{S}^\prime\rVert_F}
    \end{split}
\end{equation}
\end{proof}

\section{Experimental details}
\label{sec:appendix_exp}

\subsection{Convergence to a Constant State}
For this experiment, we consider the Cora dataset provided by Pytorch-Geometric~\cite{paszke2019pytorch}, consisting of $2708$ nodes and $5429$. For initialize $128$ layers of GAT, GCN, and GraphSAGE using their default initialization and no bias. After each layer, we use a ReLU activation. After each ReLU activation, we track both the squared norm $||\mathbf{X}||_F^2$ and the Dirichlet energy $E(\mathbf{X}) = tr(\mathbf{X}^T\mathbf{\Delta}_{rw} \mathbf{X})$ using the random walk Laplacian $\mathbf{\Delta}_{rw} = \mathbf{I}_n - \mathbf{D}^{-1}\mathbf{A}$, as the constant vector is in its nullspace, as suggested by \citet{rusch2022graph}.

\subsection{Empirical Validation}
We provide an additional experiment on synthetic data that shows that linearly independent features are not practically achievable, even for non-linear models.
\subsubsection{Synthetic Dataset}
\label{sec:synthetic}
The theory states that models utilizing a single Kronecker product suffer from rank collapse, limiting their learnability of tasks that require multiple independent features per node. Given any graph, if the feature space collapses to two or fewer dimensions, the representations of any four nodes always form a quadrilateral in that plane. 
Linear decision boundaries cannot classify points on opposite sides together. The features extracted by SKP can be linearly independent (see Theorem 6.1), so tasks of this form are trivial when accessing another dimension. 
Our experimental setting follows this idea: We randomly generate an Erdos–Rény (ER) graph, that consists of $20$ nodes and an edge probability of $0.2$. We then select four its nodes at random, denoted by $V = (v_1,v_2,v_3,v_4)$. The 3-class classification task consists of three tasks for each node, where all pairs of nodes belong to the same class exactly once. Precisely, the target for all nodes is $\mathbf{Y} = \begin{bmatrix}
    $1$ & $1$ & $1$ \\
    $1$ & $0$ & $0$ \\
    $0$ & $1$ & $0$ \\
    $0$ & $0$ & $1$ \\
    0 & 0 & 0\\
    \vdots & \ddots & \vdots \\
    0& 0 & 0
\end{bmatrix}.$

For loss, optimization, and accuracy calculation, we only consider the labels of the four nodes and ignore all remaining outputs. 

If each feature is constant across nodes, all nodes are classified the same, resulting in a $50\%$ accuracy. If the representations have rank one and all feature vectors are on a line, representations need to be adjacent in order to be classified correctly, which cannot be fulfilled for all pairs simultaneously. Since our GNNs are finite-depth, representations are not normalized, and the maximum eigenvalue may have geometric multiplicity larger than one, the models can achieve higher accuracy.

\begin{figure}[tb]
     \centering
     \begin{subfigure}[b]{0.49\textwidth}
         \centering
        \def\svgwidth{\textwidth}
         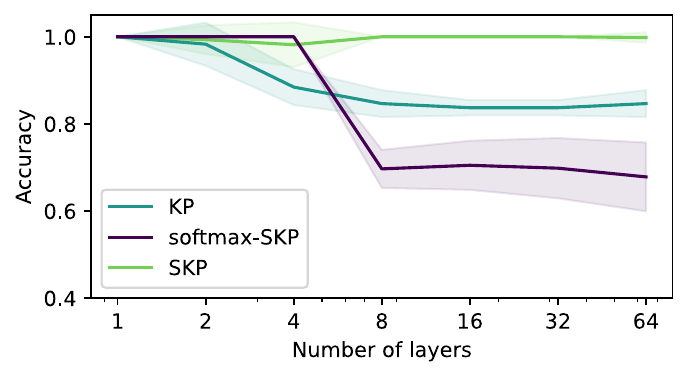
         \caption{Accuracy: mean and standard deviation.}
     \end{subfigure}
     \hfill
     \begin{subfigure}[b]{0.49\textwidth}
         \centering
        \def\svgwidth{\textwidth}
         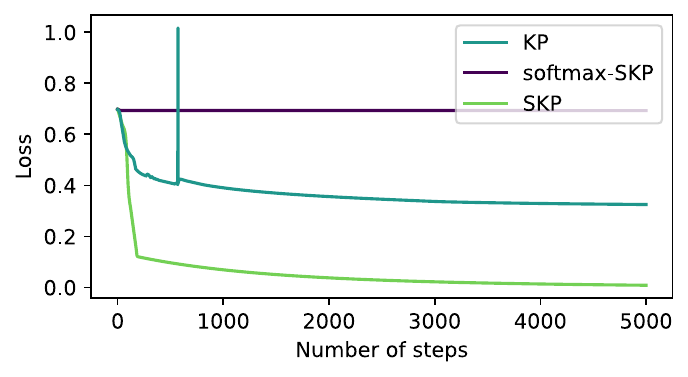
         \caption{Optimization loss for $8$ layers.}
     \end{subfigure}
        \caption{Accuracies and loss dynamics for the synthetic task.}
        \label{fig:synthetic}
\end{figure}

\paragraph{Training details}
Node representations $\mathbf{X}^{(0)}\in\mathbb{R}^{4\times 6}$ are randomly initialized from a normal distribution $x_{ij}\sim N(0,1)$ with $0$ mean and standard deviation of $1$. Aggregations and feature transformation are randomly initialized such that the norm of resulting node representations is typically not close to zero or to infinity so that the gradients do not vanish or explode. Precisely, for $\Tilde{\mathbf{A}}_1^{(k)}\in\mathbb{R}^{20\times 20}$, $\Tilde{\mathbf{A}}^{(k)}_2\in\mathbb{R}^{20\times 20}$ the weight of each generated edge $(i,j)$ is sampled from a normal distribution $\Tilde{a}_{ij} \sim \mathcal{N}(\frac{1}{|N_j|},0.05)$ with mean as one over the number of incoming nodes and standard deviation $0.05$. Similarly, the feature transformations $\mathbf{W}_1^{(k)}\in\mathbb{R}^{6\times 6}$, $\mathbf{W}^{(k)}_2\in\mathbb{R}^{6\times 6}$ are sampled from a normal distribution $\Tilde{w}_{ij} \sim \mathcal{N}(\frac{1}{3},0.05)$ with mean as one over the number of features and standard deviation $0.05$. Using $\mathbf{T}^{(k)}$ as one of the considered graph convolutions, the update function is $\vect(\mathbf{X}^{(k+1)})=\phi(\mathbf{T}^{(k)}\vect(\mathbf{X}^{(k)})$ with $\phi$ as the ReLU activation function. After $l$ iterations, we use these node representations for our three classification tasks using the affine transformation
\begin{equation}
    \hat{\mathbf{Y}} = \sigma(\mathbf{X}^{(l)}\mathbf{W}_{c} + \mathbf{b})
\end{equation}
with $\sigma$ as sigmoid activation, a feature transformation $\mathbf{W}_c\in\mathbb{R}^{6\times 3}$ and a feature-wise bias term $\mathbf{b}\in\mathbb{R}^{3}$.
We evaluated all three described update functions for $l\in[1,2,4,8,16,32,64,128]$ layers. For each method and each number of layers, variables are randomly initialized and optimized with binary cross-entropy using the Adam optimizer until the loss does not decrease for $500$ steps. 
Each experiment is executed three times for each graph, of which the best achieved accuracy is considered. We then repeat this process for $50$ random graphs and report the mean accuracy and its standard deviation.
The reproducible experiments are added as supplementary material.

We additionally run the same experimental setting for a linearized version ($\phi$ as identity function) and a version that also reuses the same aggregation across all layers $\Tilde{\mathbf{A}}_1^{(1)} = \dots = \Tilde{\mathbf{A}}_1^{(k)}$, $\Tilde{\mathbf{A}}_2^{(1)} = \dots = \Tilde{\mathbf{A}}_2^{(k)}$.
All configurations were run sequentially on a single CPU. The entire runtime was around 30 hours.

\begin{table}[]
\footnotesize
    \centering
    \begin{tabular}{c c c c c c c c c c}
        \toprule
        \# of layers & $1$ & $2$ & $4$ & $8$ & $16$ & $32$ & $64$ & $128$ \\
        \midrule
        $\mathrm{softmax-SKP}_{HL}$ & $100\pm0$ & $100\pm0$ & $82\pm19$ & $71\pm6$ & $72\pm6$ & $70\pm6$ & $67\pm7$ & $63\pm8$ \\
        $\mathrm{softmax-SKP}_L$ & $100\pm0$ & $100\pm0$ & $100\pm0$ & $70\pm4$ & $71\pm6$ & $70\pm7$ & $68\pm8$ & $65\pm8$ \\
        softmax-SKP & $100\pm0$ & $100\pm0$ & $100\pm0$ & $70\pm4$ & $71\pm6$ & $70\pm7$ & $68\pm8$ & $65\pm8$ \\
        $\mathrm{KP}_{HL}$ & $100\pm0$ & $100\pm0$ & $83\pm0$ & $83\pm0$ & $83\pm0$ & $83\pm0$ & $83\pm0$ & $82\pm4$ \\
        $\mathrm{KP}_L$ & $100\pm0$ & $100\pm0$ & $83\pm0$ & $83\pm0$ & $83\pm0$ & $83\pm0$ & $83\pm0$ & $82\pm4$ \\
        KP & $100\pm0$ & $98\pm5$ & $88\pm4$ & $85\pm3$ & $84\pm2$ & $84\pm2$ & $85\pm3$ & $82\pm4$ \\
        $\mathrm{SKP}_{HL}$ & $100\pm0$ & $100\pm0$ & $100\pm0$ & $100\pm0$ & $100\pm0$ & $100\pm0$ & $100\pm0$ & $94\pm12$\\
        $\mathrm{SKP}_L$ & $100\pm0$ & $100\pm0$ & $100\pm0$ & $100\pm0$ & $100\pm0$ & $100\pm0$ & $100\pm0$ & $96\pm10$\\
        SKP & $100\pm0$ & $99\pm3$ & $98\pm5$ & $100\pm0$ & $100\pm0$ & $100\pm0$ & $100\pm1$ & $95\pm10$\\
        \bottomrule
    \end{tabular}
    \caption{Maximum and mean$\pm$standard deviations accuracies for all considered scenarios. Subscript $L$ denotes the linearized version, subscript $HL$ denotes the version that additionally uses homogeneous aggregation matrices.}
    \label{tab:num_results}
\end{table}

\paragraph{Detailed results}
We provide additional numerical results in Figure~\ref{fig:synthetic} and Table~\ref{tab:num_results}, including the standard deviation of all settings. If the representations were constant, only a $50\%$ accuracy could be achieved, so there is slightly more relevant information in all features, even for the softmax-SKP. However, the representations are quickly converging to a low-rank state, so both softmax-SKP and KP do not solve this task a single time for eight or more layers. All models are slightly more unstable with increased depth, which is mainly a vanishing or exploding gradient issue. As this was historically an issue in other domains, similar methods can be used to solve this, e.g., normalization layers.

\end{document}